\documentclass{article}
\usepackage[preprint]{neurips_2026}

\usepackage[utf8]{inputenc} % allow utf-8 input
\usepackage[T1]{fontenc}    % use 8-bit T1 fonts
\usepackage{hyperref}       % hyperlinks
\usepackage{url}            % simple URL typesetting
\usepackage{booktabs}       % professional-quality tables
\usepackage{amsfonts}       % blackboard math symbols
\usepackage{nicefrac}       % compact symbols for 1/2, etc.
\usepackage{microtype}      % microtypography
\usepackage{xcolor}         % colors
\usepackage{subfiles}
\usepackage{graphicx}
\usepackage{subcaption}
\usepackage{enumitem}
\usepackage{amsmath}
\usepackage{amssymb}
\usepackage{mathtools}
\usepackage{amsthm}
\usepackage{colortbl}
\definecolor{avggray}{gray}{0.92}
\usepackage{tcolorbox}
\tcbuselibrary{theorems}
\usepackage{pifont}
\usepackage{multirow}
\usepackage{wrapfig}
\usepackage{makecell}
\usepackage{tikz}
\newcommand*\circled[1]{\tikz[baseline=(char.base)]{%
  \node[shape=circle,draw,inner sep=1pt] (char) {\footnotesize #1};}}
\usepackage[ruled,vlined,linesnumbered,commentsnumbered]{algorithm2e}
\SetKwComment{tcp}{/* }{ */}

\SetCommentSty{bluecomment}
\newtcbtheorem[number within=section]{theorem}{Theorem}{
  colback=gray!5, 
  colframe=gray!35!black, 
  fonttitle=\bfseries,
  boxrule=0.5mm
}{th}
\newtcbtheorem[number within=section]{assumption}{Assumption}{
  colback=white, 
  colframe=black, 
  fonttitle=\bfseries,
  boxrule=0.2mm,
  sharp corners
}{ass}
\newtcbtheorem[]{restatedtheorem}{Theorem}{
  colback=gray!5, 
  colframe=gray!35!black, 
  fonttitle=\bfseries,
  boxrule=0.5mm,
  attach boxed title to top left={xshift=0.5cm,yshift=-2mm},
  boxed title style={colback=gray!35!black}
}{rth}
\newtcbtheorem[use counter from=theorem]{lemma}{Lemma}{
  colback=gray!5, 
  colframe=gray!35!black, 
  fonttitle=\bfseries,
  boxrule=0.5mm
}{lem}
\usepackage{tcolorbox}
\tcbuselibrary{breakable, skins}
\definecolor{sffblue}{RGB}{31, 78, 121}    % deep navy
\definecolor{sffblue!10}{RGB}{220, 232, 242} % light tint for background
\tcbset{
  takeaway/.style={
    enhanced,
    colback=sffblue!7,        % very light blue tint
    colframe=sffblue,         % deep navy border
    boxrule=0.9pt,
    arc=4pt,                  % sharp corners — cleaner with left-bar style
    left=6pt, right=6pt, top=4pt, bottom=4pt,
    fontupper=\small,
  }
}
\newcommand{\cmark}{\textcolor{green!70!black}{\checkmark}}
\newcommand{\xmark}{\textcolor{red}{\ding{55}}}
\usepackage[capitalize,noabbrev]{cleveref}
\usepackage[disable,textsize=tiny]{todonotes}

\title{Unveiling the Depth-Performance Dilemma in Split-Federated Fine-tuning of LLMs}

\author{%
  Hariharan Ramesh\,\textsuperscript{1} \qquad
  Someshwaran Murugaiyan\,\textsuperscript{2} \qquad
  Jyotikrishna Dass\,\textsuperscript{1} \\[0.45em]
  \normalfont
  \textsuperscript{1}School of Electrical, Computing \& Software Engineering,
  University of Arizona, Tucson, AZ, USA \\
  \textsuperscript{2}Vellore Institute of Technology, Vellore, India \\[0.35em]
  \normalfont
  \texttt{hariharanr@arizona.edu}\quad 
  \texttt{someshwaran.2022@vitstudent.ac.in}\quad
  \texttt{jdass@arizona.edu} 
}
\begin{document}
\maketitle
\begin{abstract}
Split Federated Fine-tuning (SFF) is a promising paradigm for scaling Large Language Models (LLMs) by partitioning model depth between resource-constrained clients and a centralized server. While system incentives for throughput and privacy favor deep partitions, the impact of such configurations on model utility remains poorly understood. In this work, we identify and characterize the Depth-Performance Dilemma: the regime that maximizes system efficiency is precisely where fine-tuning quality collapses. Through a comprehensive audit across four model scales (GPT-2 to Llama-3-8B) and diverse benchmarks, we demonstrate that deeper partitions provide monotonic gains in throughput and privacy at the cost of catastrophic performance plateaus. We evaluate a suite of state-of-the-art federated adapter aggregation methods including \textsc{Avg}, \textsc{Stack}, \textsc{SVD}, and \textsc{Freeze}, revealing that while these techniques are effective in standard Federated Learning, they fail to mitigate the artifacts unique to split architectures. Finally, we provide a mechanistic diagnosis for this failure, tracing the collapse to the near-isometric topology of Transformers, which allows aggregation noise to propagate without attenuation until it triggers Attention Collapse in the server partition. Our findings challenge the prevailing assumption that partition depth is a utility-neutral tuning knob and provide a structural foundation for stable distributed LLM fine-tuning.
\end{abstract}

\section{Introduction}

Large Language Models (LLMs) face a fundamental deployment challenge: privacy-sensitive environments necessitate Federated Learning (FL) \cite{pmlr-v54-mcmahan17a}, yet on-device fine-tuning remains computationally prohibitive even with Parameter-Efficient Fine-Tuning (PEFT) like LoRA \cite{DBLP:journals/corr/abs-2106-09685}. Furthermore, fully decentralized training often suffers from the statistical instability of non-IID client drift. Split-Federated Fine-tuning (SFF) \cite{Thapa_Mahawaga_Arachchige_Camtepe_Sun_2022} resolves this by partitioning the model at a cut layer: clients execute shallow layers in parallel while a shared Main Server sequentially processes the remaining deep blocks (illustrated in Figure~\ref{fig:SFF}). By synchronizing client-side adapters via a Federated Server, SFF enables memory-feasible, privacy-aware optimization across distributed datasets.

Historically, split architectures were validated on Convolutional Neural Networks (CNNs), where pooling layers create an information "funnel" \cite{10.1145/3065386} that attenuates representations at depth. This property makes partition depth a convenient efficiency knob in CNN-based SFL, with manageable performance trade-offs. Recent LLM frameworks like SplitLoRA \cite{lin2024splitlorasplitparameterefficientfinetuning} and HSplitLoRA \cite{lin2025hsplitloraheterogeneoussplitparameterefficient} implicitly extrapolate this intuition, assuming deep cuts remain "safe" for Transformers. However, Transformers maintain uniform dimensionality across all layers, lacking the spatial downsampling that stabilizes deep partitions in CNNs. Despite this structural divergence, the relationship between cut-layer depth and SFF performance for LLMs remains systematically uncharacterized.

\begin{wrapfigure}{r}{0.5\textwidth}
    \centering
    \vspace{-1em}\includegraphics[width=1\linewidth]{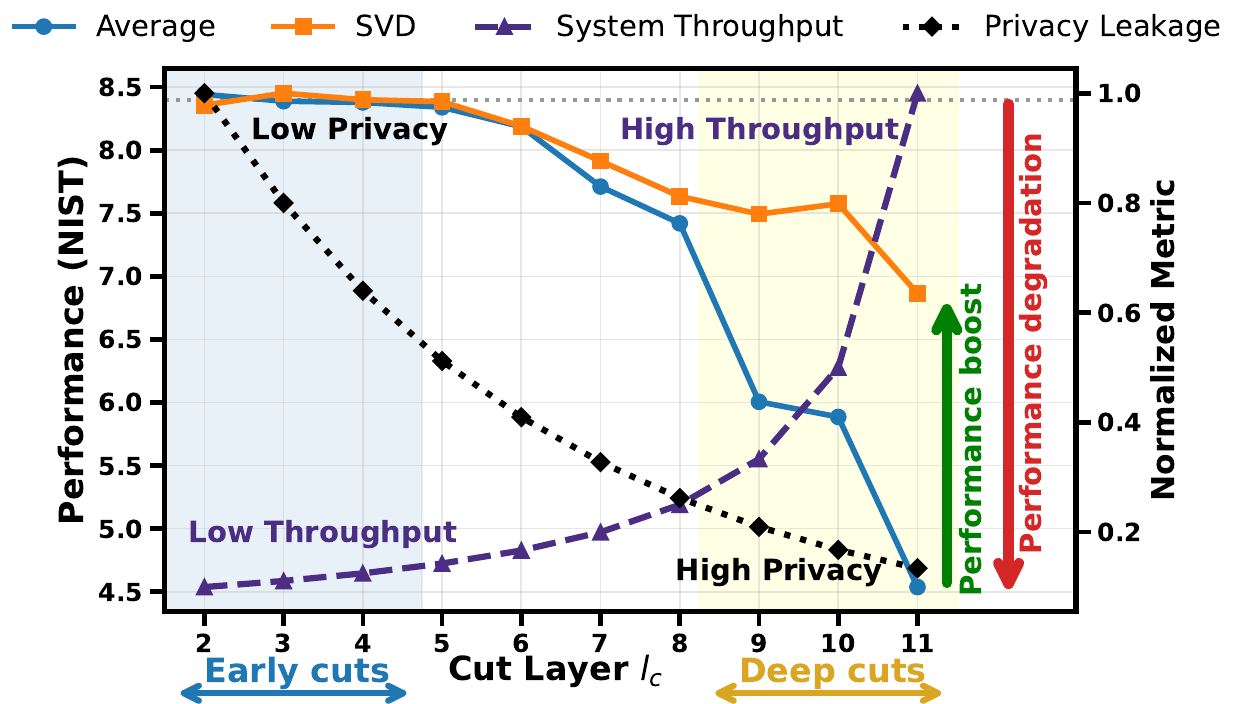}
    \caption{\textbf{Depth-Performance Dilemma.} As partition depth
    (cut layer $\ell_c$) increases, system throughput rises and privacy leakage
    decays, yet task performance (NIST, left axis) degrades consistently.
    Throughput and normalized privacy leakage (right axis) are reported alongside
    performance for GPT-2 Small under heterogeneous clients. The throughput curve
    corresponds to the server-bound regime, where the throughput-optimal depth
    $d^{\star}{=}K/(K+\rho_s/\rho_c)$ meets or exceeds the deepest cut so the
    server is the bottleneck at every $\ell_c$; see Appendix~\ref{app:throughput}.}
    %as functions of $\ell_c$ for GPT-2 Small under heterogeneous clients.} The operating regime that is optimal for system efficiency and privacy is precisely where fine-tuning quality collapses.}
    \label{fig:depth_dilemma}
    %\vspace{-5pt}
\end{wrapfigure}

In this work, we conduct a comprehensive empirical evaluation of SFF across diverse model scales and benchmarks. Through exhaustive cut-layer sweeps and multiple aggregation strategies, we identify a central \textbf{Depth-Performance Dilemma} (Figure~\ref{fig:depth_dilemma}): \textit{while deepening the cut layer increases system throughput and reduces privacy leakage, model performance degrades consistently across every tested configuration}. Figure~\ref{fig:depth_dilemma} makes this tension concrete for GPT-2 Small: as the cut layer $\ell_c$ deepens from early to deep cuts, system throughput rises steeply and privacy leakage decays toward zero (right axis), while task performance (NIST, left axis) falls consistently in the opposite direction. Because these two system-favorable trends move against the utility curve, the operating regime most attractive for system efficiency and privacy is precisely where fine-tuning quality collapses. We trace this collapse to a mismatch between federated aggregation and Transformer topology. While shallow cuts allow the server to absorb aggregation noise, deep cuts propagate noise through the residual architecture into ``collapsing'' attention layers in the server tail. This establishes the dilemma as intrinsic property of split-LLM under heterogeneity rather than a mere optimization failure.

\textbf{Contributions.} In summary, our contributions are as follows:
\begin{enumerate}[nosep, leftmargin=*]
    \item \textit{Characterization of the Depth-Performance Dilemma:} We identify a fundamental tension in SFF where partition depths that maximize throughput and privacy consistently collapse fine-tuning quality. We show this trade-off holds consistently across four model scales (GPT-2 Small/Med/Large, Llama-3-8B) and three diverse benchmarks (E2E NLG, GLUE, GSM8K).% validated via adversarial MLP inversion attacks and system profiling.

    \item \textit{Systematic Aggregation Audit:} We provide the first formal evaluation of \textsc{Average}~\cite{zhang2023towards}, \textsc{Freeze}~\cite{wang2023privateloraefficientprivacypreserving}, \textsc{Stack}~\cite{flora}, and \textsc{SVD}~\cite{bai2024federated} in the split-federated setting. We characterize how their unique noise profiles interact with partition depth and client heterogeneity across exhaustive cut-layer sweeps.

    \item \textit{Mechanistic Diagnosis:} Using perturbation propagation and effective rank analysis, we trace the collapse to three compounding structural factors: aggregation noise, near-isometric noise propagation through the Transformer's residual architecture, and attention collapse in deep server layers that eliminates the nonlinear capacity required for error correction.

\end{enumerate}

\section{Related Work and Gaps}

We position our work at the intersection of Split Learning (SL), LLM fine-tuning, and Federated aggregation theory, identifying a gap in noise-resilient aggregation for deep-cut Transformer topologies. Table~\ref{tab:comparison} summarizes how SFF differs from existing paradigms.

\begin{table*}[t!]
\caption{\textbf{The SFL Landscape.} Comparison of Split Learning approaches. While CNN-based works have extensively analyzed role of cut layer depth, LLM-based works have largely treated the cut layer as a static efficiency knob or privacy boundary. \textbf{Ours} \textbf{is the first to systematically audit the impact of deep partitions in Transformers performance across various aggregation strategies}.
(\(D\) denotes cut layer depth and \(P\) denotes performance;
arrows indicate trend direction trends as depth increases.)}
\label{tab:comparison}

\centering
\small
\renewcommand{\arraystretch}{1.2}
\setlength{\tabcolsep}{4pt}
\begin{sc}
\resizebox{\textwidth}{!}{
\begin{tabular}{l c c c c} % c}
\toprule
%{\scshape \centering{Works}} &
\multicolumn{1}{c}{\scshape Works} &
{\scshape Model} &
{\scshape Aggregation} &
{\scshape Depth Analysis} &
{\scshape LoRA Rank Hetero.} \\ %&
%{\scshape Convergence Analysis} \\
\midrule
SFL \cite{Thapa_Mahawaga_Arachchige_Camtepe_Sun_2022} & CNNs & \cmark\ (Avg) & \xmark & N/A \\ % & \cmark \\
Analysis~\cite{10.5555/3737916.3741203} & CNNs & \cmark\ (Avg) & $D \uparrow P \uparrow$ & N/A \\ 
% & \cmark \\
Analysis~\cite{dachille2024impactcutlayerselection} & CNNs & \cmark\ (Avg) & \ $D \uparrow P \downarrow$ & N/A \\ 
%  & \cmark \\
\midrule
SplitLoRA~\cite{lin2024splitlorasplitparameterefficientfinetuning} & LLMs & \cmark\ (Avg) & \xmark & \xmark \\ % & \xmark \\
SplitFrozen~\cite{ma2025splitfrozensplitlearningdeviceside} & LLMs & \cmark\ (Avg) & \xmark & \xmark \\ % & \xmark \\
FSL-SAGE~\cite{nair2025fslsage} & CNN/LLM & \cmark\ (Avg) & \xmark & \xmark \\ % & \cmark \\
HSplitLoRA~\cite{lin2025hsplitloraheterogeneoussplitparameterefficient} & LLMs & \cmark\ (Stack) & \xmark & \cmark \\ % & \cmark \\
\midrule
\rowcolor{gray!15}
\textbf{Depth-Performance Dilemma (Ours)} & \textbf{LLMs} & \textbf{Average, Freeze, Stack, SVD} &
$D \uparrow P \downarrow$ & \textbf{\cmark} \\ % & \textbf{\cmark} \\
\bottomrule
\end{tabular}
}
\end{sc}
\end{table*}

\textbf{Split Federated Learning (SFL).} SL~\cite{DBLP:journals/corr/abs-1812-00564} decouples model depth from client constraints via a cut layer. Originally validated on CNNs, SL relies on spatial pooling to reduce bandwidth and filter noise. SFL extensions (SFL-V1/V2)~\cite{Thapa_Mahawaga_Arachchige_Camtepe_Sun_2022} optimize synchronization but retain the structural premise that partition depth is a tunable knob for efficiency, ostensibly independent of model topology. While CNN-based audits~\cite{dachille2024impactcutlayerselection} and theoretical analyses~\cite{10.5555/3737916.3741203} suggest depth-invariance, these insights are coupled to CNN-specific inductive biases (spatial downsampling) that are absent in Transformers.

\begin{figure}[t]
    \centering
    \includegraphics[width=0.85\linewidth]{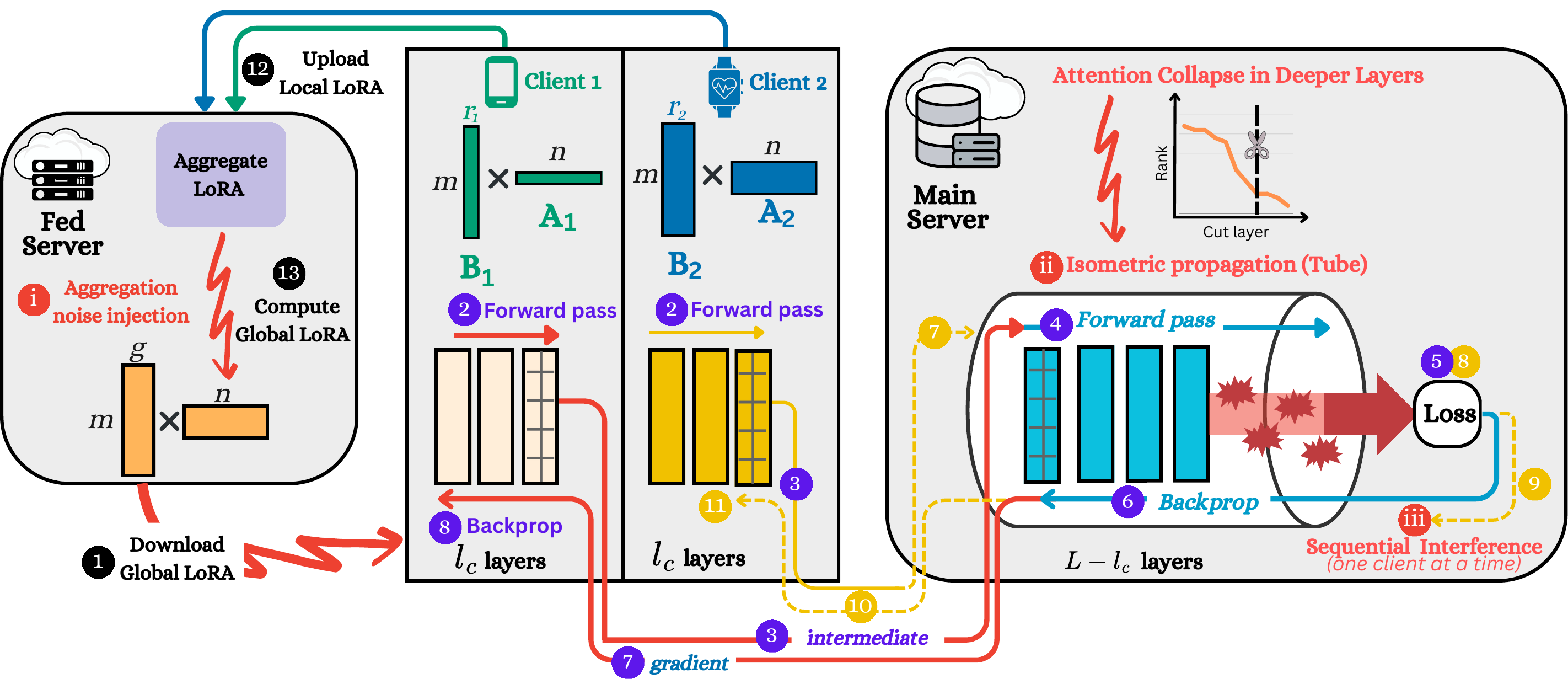}
    \caption{\textbf{Split-Federated Fine-tuning of LLMs.} Clients parallelize the shallow partition while the Main Server sequentially processes the shared deep blocks. Fine-tuning performance degrades consistently with cut-layer depth due to: (i) aggregation noise from heterogeneous LoRA updates; (ii) near-isometric propagation where residual connections transmit noise unattenuated; and (iii) attention collapse where deep server layers lack the rank and nonlinear capacity to correct errors.}
    \label{fig:SFF}
    \vspace{-1em}
\end{figure}

\textbf{Split LLMs.} Recent adaptations of SL to Transformers prioritize system feasibility over topological (structural) stability. Enabling frameworks like SplitLoRA~\cite{lin2024splitlorasplitparameterefficientfinetuning} and HSplitLoRA~\cite{lin2025hsplitloraheterogeneoussplitparameterefficient} implement distributed fine-tuning, while optimizations like SplitFrozen~\cite{ma2025splitfrozensplitlearningdeviceside} and SplitLLM~\cite{Chen_2024} focus on reducing device load or inference latency. Privacy-centric works~\cite{10.24963/ijcai.2025/57, wang2023privateloraefficientprivacypreserving} treat the cut layer strictly as a security boundary. Crucially, these studies tacitly treat partition depth as utility-neutral, ignoring how the constant-width near-isometric structure of a Transformer propagates aggregation artifacts.

\textbf{Federated Adapter Aggregation.} Aggregating LoRA adapters presents a unique challenge: standard averaging (\textsc{Avg}) is mathematically ill-suited for factorized matrices ($\mathbf{W + BA}$) due to the non-linearity of the product: $\mathbb{E}[\mathbf{B}]\mathbb{E}[\mathbf{A}] \neq \mathbb{E}[\mathbf{BA}]$~\cite{zhang2023towards, cho-etal-2024-heterogeneous}. While FL-centric solutions like \textsc{Stack}~\cite{flora}, \textsc{Freeze}~\cite{sun2024improving}, and \textsc{SVD}~\cite{bai2024federated, ramesh2025floristsingularvaluethresholding} aim to restore mathematical accuracy, reduce communication overhead, or support heterogeneous ranks, they have been evaluated exclusively in standard Federated setups. Consequently, their impact on Split Federated Finetuning remains unexamined, particularly in the context of a shared server partition where aggregation artifacts directly influence the sequential processing of deep blocks.

% \noindent \textbf{Open Challenges.} Two gaps remain unaddressed. First, \textit{no existing work has examined whether the constant-width, isometric topology of Transformers invalidates the CNN-derived assumption that partition depth is a utility-neutral tuning knob}. Second, \textit{LoRA aggregation methods designed for standard FL have never been evaluated in a split setting, where aggregation noise enters a server partition of diminishing capacity as depth increases}. We address both gaps below.

\textbf{Gaps.} We summarize two main gaps as follows: (i) \textit{Topological Sensitivity:} determining whether increasing the split depth for system efficiency creates a tipping point that disrupts the LLM's performance, and (ii) \textit{LoRA Aggregation Impact:} While methods exist to combine client updates under federated setup, they have not been evaluated and compared in a split-model context with a shared server partition with sequential updates.

%attempt to be mathematically accurate, improve communication efficiency with half the number of adapters, and support heterogeneous LoRA ranks, respectively, they have been evaluated solely in standard Federated fine-tuning setups. Hence, their impact on performance of SFF of LLMs remains unexamined especially with shared server model partition with sequential processing.

%without split LLM architecture. In SFF, this aggregation noise immediately corrupts the forward pass of the sequential server partition, a phenomenon we identify as the fundamental driver of performance collapse.
\section{The Depth-Performance Dilemma}
\label{sec:dilemma}

To address the above gaps, we begin by establishing the operational case for deep partitioning in SFF
(Section~\ref{sec:incentives}), then demonstrate empirically that the very
depths favored by system incentives are precisely where fine-tuning quality
collapses (Section~\ref{sec:dilemma_finding}). This tension, shown in Figure~\ref{fig:depth_dilemma}, which we term
the \textbf{Depth-Performance Dilemma}, motivates the diagnostic
investigation in Section~\ref{sec:analysis}.
 
%==========================================================================
\subsection{System Benefits of Deep Cuts}
\label{sec:incentives}
%==========================================================================
 
An overview of the SFF pipeline is shown in Figure~\ref{fig:SFF}; full
pseudocode is in Appendix~\ref{app:pseudocode}. We consider $N$ clients
$\mathcal{C}=\{1,\dots,N\}$, each holding a private dataset $\mathcal{D}_i$.
The model is partitioned at cut layer $\ell_c$: client $i$ executes the
shallow sub-model $f_C(\cdot;\Theta_{C,i})$ through the first $\ell_c$
Transformer blocks, while a shared \emph{Main Server} executes the remaining
deep sub-model $f_S(\cdot;\Theta_S)$. Clients transmit cut-layer activations
to the server and receive activation gradients in return, following the
parallel execution model of
SplitLoRA~\cite{lin2024splitlorasplitparameterefficientfinetuning}. A separate
\emph{Federated Server} periodically aggregates client-side
LoRA~\cite{DBLP:journals/corr/abs-2106-09685} adapters and broadcasts a
global adapter back to all clients.

\textbf{Impact on Throughput.}
The primary constraint in SFF is the \emph{sequential server bottleneck}: the
Main Server processes client activations one at a time, so server-side compute
directly throttles throughput. Unlike CNNs, where pooling layers create
irregular compute profiles across depth, Transformer blocks are \emph{uniform}
in hidden dimension and cost, every layer shifted to the client yields a
precise, predictable reduction in server load. Formally
(Appendix~\ref{app:throughput}), in the server-bound regime
$d < d^{\star}{=}K/(K+\rho_s/\rho_c)$, throughput scales as
$\mathcal{T}(d) = \Theta\!\left(\frac{1}{1-d}\right)$ with $d = \ell_c/L$, so
moving from $d{=}0.5$ to $d{=}0.9$ delivers up to a $5\times$ gain. For realistic
client counts and compute ratios, $d^{\star}$ can reach the deepest cut (e.g.,
$d^{\star}{=}0.91$ at $K{=}10$, $\rho_s/\rho_c{=}1$; Table~\ref{tab:dstar}),
keeping deep cuts throughput-favorable; where it does not, the deep-cut incentive
is carried by privacy, which improves with depth independent of $K$ and compute.

\textbf{Impact on Privacy.} The cut-layer representation $\mathbf{h}_{\ell_c}$ is the only signal the
server ever observes about a client's private input. At shallow cuts these
representations are close to raw token embeddings and can be inverted with
high fidelity; at deep cuts, multi-head attention has progressively diffused
token identity into the global context. 

\begin{wrapfigure}{r}{0.45\textwidth}
    \vspace{-2em}
    \centering
    \includegraphics[width=0.9\linewidth]{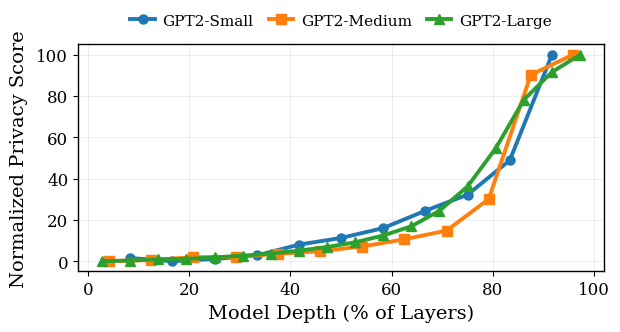}
    \caption{\textbf{Normalized Privacy Score vs.\ partition depth.}
    Score$\,{=}\,(1{-}\mathrm{RR}){\times}100$, where RR is token recovery
    rate of an adversarial MLP inversion model. Score rises sharply beyond
    ${\approx}80\%$ depth and reaches 100 at deep cuts.}
    \label{fig:ep_vs_depth}
    \vspace{-2em}
\end{wrapfigure}

We validate this directly by training
adversarial MLP inversion models that reconstruct input tokens from cut-layer
hidden states (setup in Appendix~\ref{app:eia}). We define \emph{Empirical
Privacy} EP$\,{=}\,(1{-}\mathrm{RR}){\times}100$, where RR is token recovery
rate of the strongest attacker. Figure~\ref{fig:ep_vs_depth} shows EP across GPT-2 Small, Medium, and Large.
At shallow cuts ($d < 0.4$), EP remains near zero, tokens are recoverable
with high accuracy from the cut-layer hidden states. Beyond $d \approx 0.8$,
EP rises sharply and reaches 100 at deep cuts ($d > 0.9$), indicating that
representations at this depth carry negligible recoverable information about the original input, consistent
across all three GPT-2 scales and, at 8B scale, for Llama-3-8B-Instruct
(Appendix~\ref{app:eia}, Figure~\ref{fig:ep_vs_depth_llama}).

% End of 3.1
\begin{tcolorbox}[takeaway]
\textbf{Key Takeaway:} 
Deeper partitions maximize both throughput and privacy. These dual incentives position deep partitioning as the unambiguous system-optimal regime. %Deeper partitions deliver monotonically higher throughput and stronger privacy guarantees. Both incentives point unambiguously toward deeper partitions as the system-optimal operating regime.
\end{tcolorbox}

%==========================================================================
% \subsection{The Cost of Deep Partitioning}
\subsection{Impact of Depth on Fine-tuning Performance}
\label{sec:dilemma_finding}

\paragraph{Experimental setup.}
We evaluate SFF across four models, three benchmarks, and up to 100 clients,
making this the most extensive empirical study of cut-layer depth in
split-federated LLM fine-tuning to date. We evaluate SFF across GPT-2 Small/Medium/Large~\cite{radford2019language} on
E2E NLG~\cite{novikova-etal-2017-e2e} and Llama-3-8B-Instruct~\cite{llama3}
on GSM8K~\cite{gsm8k}, with cut layers swept exhaustively ($\ell_c \in
\{1,\dots,L{-}1\}$ for GPT-2; stride-3 sweep for LLaMA; 30-client
federated setup with 3 sampled per round for 50 communication rounds). All experiments use
non-IID data (Dirichlet, $\alpha{=}0.5$), with both heterogeneous
($r_i \in \{4,8,16\}$) and homogeneous ($r{=}16$) LoRA rank configurations.
Generalization to natural language understanding is validated on
GLUE~\cite{DBLP:journals/corr/abs-1804-07461} (GPT-2 Small, 100 clients); results are in
Appendix~\ref{app:glue}. Performance is measured by PPL and task metrics
(BLEU, NIST, METEOR, ROUGE-L, CIDEr for E2E; accuracy for GSM8K and GLUE).
Full hyperparameter details and complexity comparisons are in
Appendix~\ref{app:exp_details} and Appendix~\ref{app:complexity_analysis}.
 
\paragraph{Aggregation protocols.}
Aggregation design in SFF remains largely underexplored. To date, only two
strategies have been studied in prior SFL works:
SplitLoRA~\cite{lin2024splitlorasplitparameterefficientfinetuning}, which
applies simple averaging, and
HSplitLoRA~\cite{lin2025hsplitloraheterogeneoussplitparameterefficient}, which
uses stacking to handle heterogeneous ranks. Beyond these, we conduct the
\emph{first systematic audit} of representative federated LoRA aggregation
methods in the split setting, a contribution in itself, since these methods
were designed and evaluated exclusively for standard (non-split) FL. We apply
LoRA to the query and value projection matrices of each attention block,
parameterizing updates as $\Delta\mathbf{W} = \mathbf{B}\mathbf{A}$ where
$\mathbf{B} \in \mathbb{R}^{d\times r}$, $\mathbf{A} \in \mathbb{R}^{r\times d}$.
Every $I=100$ local steps, client adapters are synchronized via one of four
strategies: \textbf{(i)~\textsc{Average}~\cite{zhang2023towards}} averages adapters
element-wise, introducing cross-term interference
($\mathbf{B}_\mathrm{agg}\mathbf{A}_\mathrm{agg} \neq \sum_i\alpha_i
\mathbf{B}_i\mathbf{A}_i$).
\textbf{(ii)~\textsc{Freeze}~\cite{sun2024improving}} fixes a shared projection
$\mathbf{A}_\mathrm{fixed}$ and trains only $\mathbf{B}$, eliminating
cross-term noise at the cost of halved adaptation capacity.
\begin{wrapfigure}{r}{0.45\textwidth}
    \vspace{-1em}
    \centering
    \includegraphics[width=\linewidth]{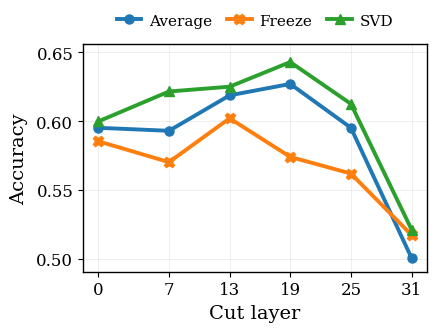}
    \caption{\textbf{GSM8K accuracy vs.\ cut-layer depth} for Llama-3-8B-Instruct. %under three aggregation methods. Performance degrades substantially at deep cuts across all methods.
    }
    \label{fig:llama_acc}
    \vspace{-3em}
\end{wrapfigure}
\textbf{(iii)~\textsc{Stack}~\cite{flora}} concatenates adapters, preserving
all client subspaces exactly but scaling the global rank as $N\times r$.
\textbf{(iv)~\textsc{SVD}~\cite{bai2024federated}} (following FlexLoRA)
aggregates in update space, $\Delta\mathbf{W}_\mathrm{agg} = \sum_i
\alpha_i\mathbf{B}_i\mathbf{A}_i$, applies SVD, and delivers
client-specific global adapters
$\mathbf{B}_k^{(g)} = \mathbf{U}_{[:,{:}r_k]}\boldsymbol{\Sigma}_{[{:}r_k,{:}r_k]}$,
$\mathbf{A}_k^{(g)} = \mathbf{V}^\top_{[{:}r_k,:]}$, truncated to each
client's rank $r_k$. Methods requiring aligned factors apply rank alignment via
zero-padding following HetLoRA~\cite{cho-etal-2024-heterogeneous}.
\paragraph{Depth degrades performance across all aggregation methods.}
Figure~\ref{fig:gpt2_s_m_grid} reports generation metrics across cut layers
for GPT-2 Small and Medium; Figure~\ref{fig:ppl_vs_cutlayer} reports validation
PPL for all three GPT-2 scales; and Figure~\ref{fig:llama_acc} reports GSM8K
accuracy for Llama-3-8B-Instruct. Across every model, benchmark, and aggregation
method, performance at deep cuts is substantially lower than at shallow cuts. GSM8K accuracy plateaus at shallow-to-moderate cuts before declining at deep cuts, consistently across all aggregation methods. Across both GPT-2 and LLaMA, the deep-cut regime, which maximizes throughput and privacy, consistently incurs the largest performance penalty. The degradation is not uniform across methods: \textsc{Freeze} collapses most sharply at deep cuts due to its constrained adaptation subspace; \textsc{Average} degrades steadily; and \textsc{Stack} and \textsc{SVD} are more robust, though neither is immune. The key observation is that \emph{no aggregation strategy escapes depth-induced degradation}, the dilemma is not an artifact of any single method's design. The rate and severity of collapse does, however, depend on the aggregation rule, a distinction we
pursue in Section~\ref{sec:analysis}.

Figure~\ref{fig:convergence} confirms this is not a convergence failure:
training loss converges in every configuration, but deeper cuts converge to
worse optima. The shallow--deep separation in PPL emerges within the first few
communication rounds and widens throughout training.
 
\begin{figure}[t]
  
  \centering  
  \vspace{-1em}\includegraphics[width=\textwidth]{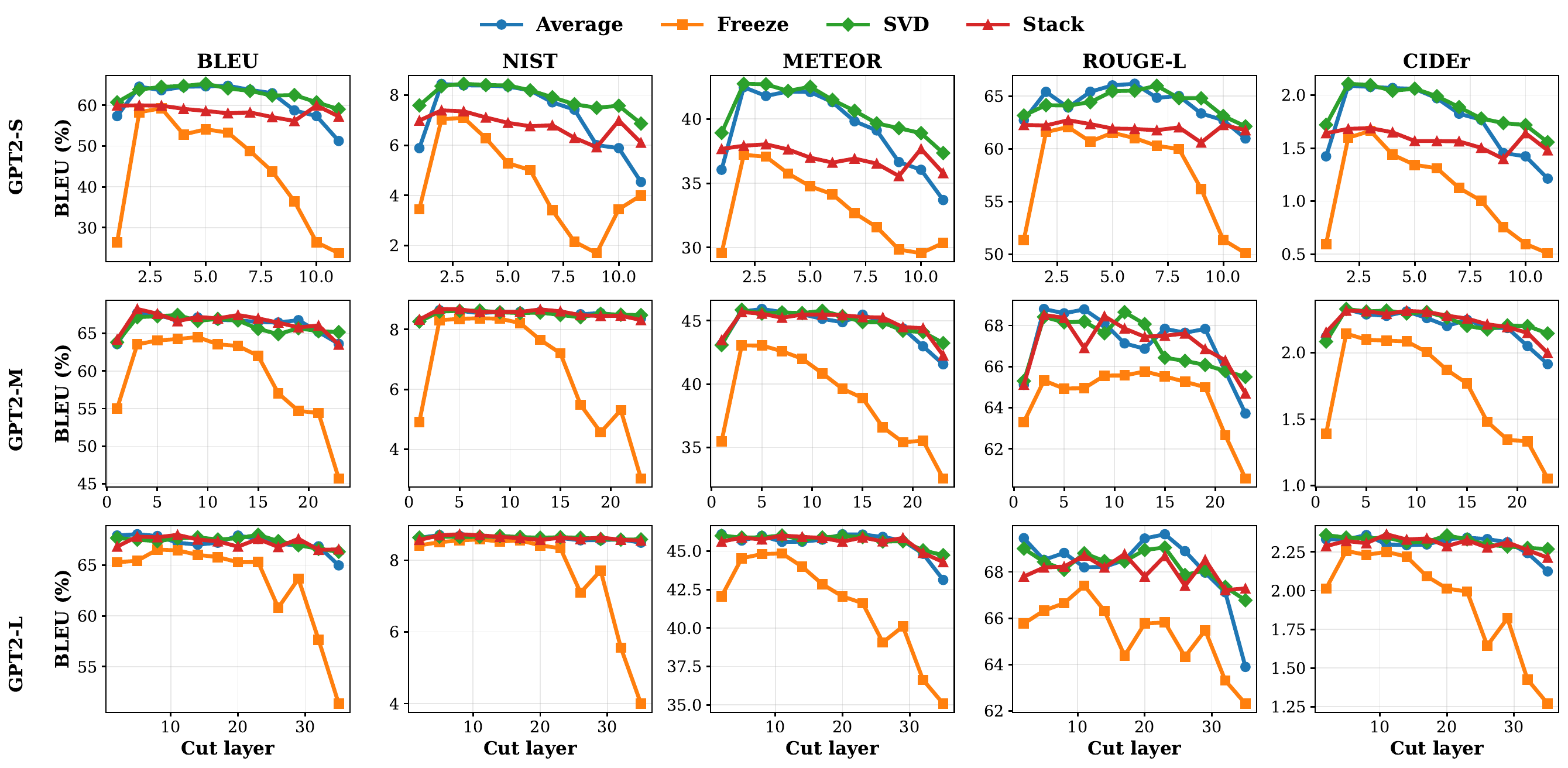}
  \caption{\textbf{Generation quality across cut-layer depth and aggregation
  method.} GPT-2 Small (top) and Medium (bottom) performance across cut layers
  for five metrics (BLEU, NIST, METEOR, ROUGE-L, CIDEr; BLEU/METEOR/ROUGE-L
  in \%). Across all methods, performance at deep cuts is substantially lower
  than at shallow cuts; the rate of collapse varies by method but no method is
  immune. The performance cliff aligns with the attention collapse thresholds
  identified in Figure~\ref{fig:rank_collapse}.}
  \label{fig:gpt2_s_m_grid}
\end{figure}

\begin{figure}[t]
    \centering
    \begin{subfigure}[t]{0.48\linewidth}
        \centering
        \includegraphics[width=\linewidth]{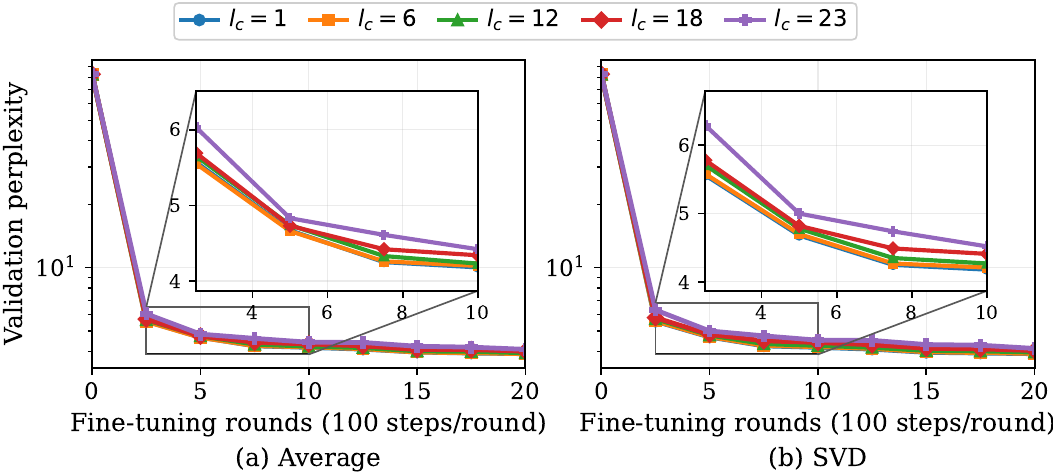}
        \caption{GPT-2 M - Homogeneous ($r{=}8$)}
        \label{fig:convergence_homo}
    \end{subfigure}
    \hfill
    \begin{subfigure}[t]{0.48\linewidth}
        \centering
        \includegraphics[width=\linewidth]{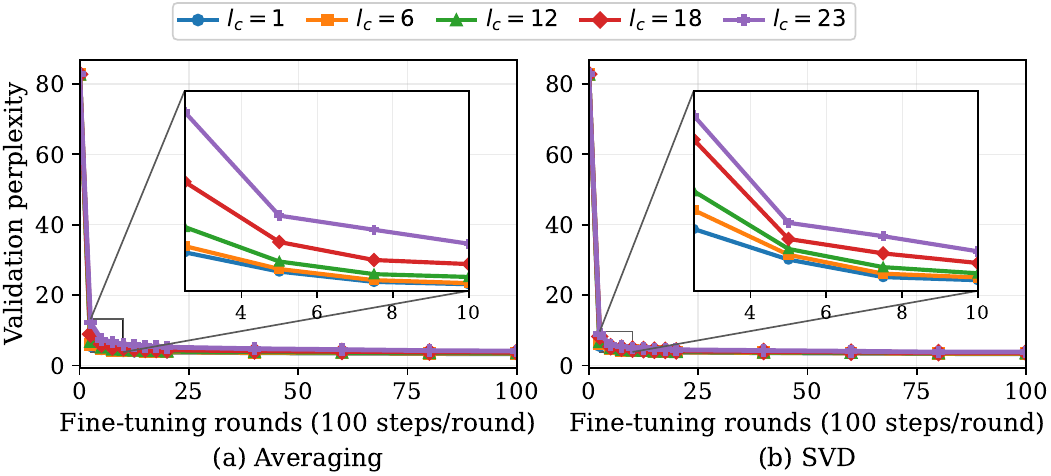}
        \caption{GPT-2 M - Heterogeneous}
        \label{fig:convergence_heter}
    \end{subfigure}
    \caption{\textbf{Convergence under increasing cut-layer depth.}
    Validation perplexity (PPL; lower is better) across SFF rounds
    (100 steps/round) for GPT-2 Medium under (a) homogeneous and
    (b) heterogeneous LoRA rank configurations. Main panels use log-$y$ scale;
    insets zoom early training with linear-$y$. Deeper partitions exhibit slower convergence and higher plateaued PPL. This degradation is amplified under heterogeneity, identifying federated aggregation as the primary failure mode. %Deeper cuts converge more slowly and plateau at substantially higher PPL; the effect is markedly strongerunder heterogeneity, implicating federated aggregation as the dominant failure mode.
    }
    \label{fig:convergence}
\end{figure}

% End of 3.2
\begin{tcolorbox}[takeaway]
\textbf{Key Takeaway:} 
%Section~\ref{sec:incentives} establishes that throughput and privacy both favor deeper partition. Section~\ref{sec:dilemma_finding} shows that model performance degrades across depth. This is the \textbf{Depth-Performance Dilemma}: the system-optimal operating regime is precisely where fine-tuning quality collapses. Section~\ref{sec:analysis} investigates why.
Section~\ref{sec:incentives} demonstrates that system throughput and data privacy both improve with deeper partitioning. Conversely, Section~\ref{sec:dilemma_finding} reveals that fine-tuning performance degrades consistently across depth. This establishes the \textbf{Depth-Performance Dilemma}: the system-optimal operating regime is precisely where model quality collapses. Section~\ref{sec:analysis} provides a mechanistic investigation into the structural causes of this failure.
\end{tcolorbox}

% These observations raise three concrete questions that guide the analysis in
% Section~\ref{sec:analysis}:
% \textbf{[RQ1]} Does the choice of aggregation method determine the severity of
% depth-induced collapse, or is degradation intrinsic to the partition regardless
% of how adapters are combined?
% \textbf{[RQ2]} If aggregation noise drives the collapse, how does it propagate
% through the server partition, does the Transformer attenuate or preserve it?
% \textbf{[RQ3]} What structural property of deep server layers prevents recovery
% from incoming noise, and what compounds this failure at depth?

% \subfile{4-analysis-v1}
%==========================================================================
\section{Diagnosing the Depth-Performance Dilemma}
\label{sec:analysis}
%==========================================================================

The experiments in Section~\ref{sec:dilemma} establish that deep cuts
consistently produce the worst performance across all models and aggregation
methods, and that the severity of collapse varies by aggregation strategy. We investigate three research questions, each building
naturally on the previous:

\smallskip
\noindent\textbf{[RQ1].} \textbf{\textit{Is performance collapse driven by
aggregation method, or is it intrinsic to partition depth?}} 
In Section~\ref{sec:q1}, we compare all four strategies under identical conditions. We observe two distinct regimes: at shallow cuts, aggregation choice has negligible impact; at deep cuts, methods diverge substantially and consistently across all models and tasks.  These findings demonstrate that the specific mechanism utilized to combine adapters assumes critical importance at greater depths, raising the fundamental question of what constitutes this operational dynamic.

This establishes that something about
\emph{how adapters are combined} becomes critically important at depth, and
naturally asks what that something is.

\noindent\textbf{[RQ2].} \textbf{\textit{Why do aggregation methods diverge at deep
cuts but not shallow ones?}}\\ Section~\ref{sec:q2} answers this by characterizing
the noise each method produces when combining heterogeneous adapters, then
showing that the Transformer server, unlike a CNN, propagates this noise
intact from cut point to output rather than attenuating it. Shallow cuts
tolerate any noise level because the long server partition provides a buffer;
deep cuts do not, because the buffer has nearly vanished.

\noindent\textbf{[RQ3].} \textbf{\textit{Why does the server lack the capacity to
correct incoming noise at depth?}}\\ Section~\ref{sec:q3} shows that deep server
layers degenerate into near-identity mappings, a phenomenon we observe
empirically via effective rank analysis, eliminating the nonlinear processing
capacity that would otherwise provide implicit error correction. Sequential
server updates compound this by forcing client gradients into a low-dimensional
manifold where they inevitably conflict.
\smallskip

%==========================================================================
\subsection{Impact of Aggregation on Depth-Induced Degradation}
\label{sec:q1}
%==========================================================================

To address \textbf{[RQ1]}, we isolate the aggregation method as the sole
variable: model, dataset, client count, and cut layer are held fixed across
conditions. Server topology and sequential update dynamics are therefore
identical; any performance divergence must be attributed to how adapters are
combined at the Federated Server.

Figure~\ref{fig:ppl_vs_cutlayer} and Table~\ref{tab:gpt2_shallow_deep_metrics}
report results across all cut-layer depths for GPT-2 Small, Medium, and Large,
with a Federated Learning (FL) baseline included for reference, representing
the case where the entire model is aggregated at every round with no server
partition.

\paragraph{Shallow cuts: aggregation-invariant regime.}
At shallow cuts, aggregation method has little impact on final performance.
For GPT-2 Small, \textsc{Average}, \textsc{Stack}, and \textsc{SVD} cluster
tightly with AVG5 between 0.564 and 0.617; similar groupings appear for Medium
(0.661--0.665) and Large (0.675--0.680). When the server retains many layers,
the residual depth is sufficient to absorb whatever the Federated Server
delivers, regardless of how adapters were combined.

\paragraph{Deep cuts: aggregation-sensitive regime.}
At depth, the picture changes decisively. Performance diverges substantially
across methods and the divergence grows with $\ell_c$. For GPT-2 Small:
\textsc{SVD} achieves AVG5~$=0.570$, \textsc{Stack} $=0.528$,
\textsc{Average} $=0.498$, \textsc{Freeze} $=0.310$. GPT-2 Medium and Large
follow the same ordering, as does Llama-3-8B-Instruct on GSM8K
(Table~\ref{tab:llama_results}). \textit{If the choice of aggregation method
were irrelevant at depth, all methods would degrade identically, they do
not.} The systematic ordering across all models and tasks confirms that
aggregation is a controllable driver of the collapse. What makes one method
degrade less than another is the focus of Section~\ref{sec:q2}.

\begin{wraptable}{r}{0.45\textwidth}
\vspace{-6pt}
\centering
\small
\caption{\textbf{Llama-3-8B GSM8K accuracy.}
Shallow and deep cuts correspond to $l_c=1$ and $l_c=31$; FedB is a standard Federated Learning baseline with no server partition.}
\label{tab:llama_results}
\setlength{\tabcolsep}{4pt}
\begin{tabular}{lccc}
\toprule
\textbf{Method} & \textbf{Shallow} & \textbf{Deep} & \textbf{FedB} \\
\midrule
\textsc{Average} & \cellcolor{green!50}0.595 & 0.500 & 0.420 \\
\textsc{Freeze}  & \cellcolor{green!50}0.585 & 0.517 & 0.420 \\
\textsc{SVD}     & \cellcolor{green!50}\textbf{0.600} & \textbf{0.521} & \textbf{0.428} \\
\bottomrule
\end{tabular}
\vspace{-6pt}
\end{wraptable}

Note that every method still degrades with depth to some degree, even
\textsc{SVD} drops from AVG5~$=0.617$ (shallow) to $0.570$ (deep) on GPT-2
Small. This residual gap, present regardless of aggregation choice, points to
a structural property of the Transformer server that no aggregation method can
address, examined in Section~\ref{sec:q3}.

\paragraph{FL baseline: the limiting case.}
Strikingly, \emph{every SFF configuration outperforms standard FL}, including
deep cuts. In FL, the entire model is subject to aggregation noise at every
round with no server partition to provide any downstream processing. In SFF,
even the deepest cut retains at least one server layer between the noise
injection point and the prediction head. The FL baseline thus represents the
true lower bound of the aggregation-noise regime, and SFF's advantage over it
narrows as the cut deepens and the server partition shrinks toward zero. This
pattern holds at 8B scale: on GSM8K, the deep-cut accuracies of $0.500$, $0.517$,
and $0.521$ for \textsc{Average}, \textsc{Freeze}, and \textsc{SVD} all exceed
the corresponding FL baseline of $0.420$, $0.420$, and $0.428$
(Table~\ref{tab:llama_results}).

\begin{tcolorbox}[takeaway]
\textbf{Key Takeaway:} Aggregation choice is irrelevant at shallow cuts but
decisive at deep cuts, where methods diverge substantially and consistently
across all models and tasks. Every SFF configuration outperforms standard FL,
confirming that even a thin server partition provides meaningful buffering.
The systematic ordering across methods raises the question of what distinguishes
them, which we answer in Section~\ref{sec:q2}.
\end{tcolorbox}

\begin{table*}[t]
\centering
\small
\caption{\textbf{GPT-2 performance at shallow and deep cuts, with FL baseline.}
Shallow (resp.\ deep) values average over the first (resp.\ last) 25\% of cut
positions. FedB is a standard Federated Learning baseline where the full model
is aggregated at every round with no server partition. Metrics: B=BLEU (\%),
N=NIST, M=METEOR (\%), R=ROUGE-L (\%), C=CIDEr; AVG5 is the geometric mean of
per-metric normalized scores. All SFF configurations outperform FL; performance
at deep cuts is lower than at shallow cuts for every method, but the margin over
FL narrows as the cut deepens. \textbf{Bold}: best; \underline{underline}:
second-best per model, depth, and metric.}
\label{tab:gpt2_shallow_deep_metrics}
\setlength{\tabcolsep}{4pt}
\begin{sc}
\resizebox{1\textwidth}{!}{
\begin{tabular}{
ll
rrrrr>{\columncolor{avggray}}r
rrrrr>{\columncolor{avggray}}r
@{\hspace{6pt}}
r
}
\toprule
\multirow{2}{*}{Model} & \multirow{2}{*}{Method}
  & \multicolumn{6}{c}{Shallow cut}
  & \multicolumn{6}{c}{Deep cut}
  & \multicolumn{1}{c}{FedB} \\
\cmidrule(lr){3-8}\cmidrule(lr){9-14}\cmidrule(lr){15-15}
 &  & B & N & M & R & C & AVG5
      & B & N & M & R & C & AVG5
      & \cellcolor{avggray}AVG5 \\
\midrule
\multirow{4}{*}{GPT2-S}
 & \textsc{Average} & \underline{61.88} & \underline{7.57} & \underline{40.10} & \textbf{64.02} & \underline{1.86} & \cellcolor{green!50}\underline{0.595} & 55.79 & 5.48 & 35.47 & \underline{62.35} & 1.36 & 0.498 & \cellcolor{red!25}0.435 \\
 & \textsc{Freeze}  & 48.00 & 5.86 & 34.61 & 58.35 & 1.29 & \cellcolor{green!50}0.476 & 28.86 & 3.05 & 29.92 & 52.56 & 0.62 & 0.310 & \cellcolor{red!25}0.220 \\
 & \textsc{Stack}   & 59.95 & 7.24 & 37.87 & 62.40 & 1.67 & 0.564 & \underline{57.79} & \underline{6.34} & \underline{36.34} & 61.55 & \underline{1.51} & \cellcolor{red!25}\underline{0.528} & \cellcolor{green!50}\textbf{0.568} \\
 & \textsc{SVD}     & \textbf{63.05} & \textbf{8.13} & \textbf{41.45} & \underline{63.81} & \textbf{1.97} & \cellcolor{green!50}\textbf{0.617} & \textbf{60.77} & \textbf{7.31} & \textbf{38.51} & \textbf{63.35} & \textbf{1.67} & \textbf{0.570} & \cellcolor{red!25}\underline{0.557} \\
\midrule
\multirow{4}{*}{GPT2-M}
 & \textsc{Average} & \underline{66.21} & \underline{8.51} & \textbf{44.92} & \textbf{67.50} & \underline{2.25} & \cellcolor{green!50}\underline{0.663} & \underline{65.17} & \underline{8.46} & 42.98 & 65.78 & 2.05 & 0.639 & \cellcolor{red!25}0.511 \\
 & \textsc{Freeze}  & 60.86 & 7.19 & 40.53 & 64.52 & 1.88 & \cellcolor{green!50}0.590 & 51.61 & 4.30 & 34.51 & 62.74 & 1.24 & 0.457 & \cellcolor{red!25}0.196 \\
 & \textsc{Stack}   & \textbf{66.67} & \textbf{8.56} & \underline{44.90} & \underline{67.33} & \textbf{2.26} & \cellcolor{green!50}\textbf{0.665} & 65.11 & 8.41 & \underline{43.72} & \textbf{65.95} & \underline{2.11} & \underline{0.645} & \cellcolor{red!25}\textbf{0.624} \\
 & \textsc{SVD}     & 66.07 & 8.49 & 44.83 & 67.29 & 2.24 & \cellcolor{green!50}0.661 & \textbf{65.37} & \textbf{8.50} & \textbf{43.85} & \underline{65.79} & \textbf{2.18} & \textbf{0.651} & \cellcolor{red!25}\underline{0.622} \\
\midrule
\multirow{4}{*}{GPT2-L}
 & \textsc{Average} & \textbf{67.95} & \textbf{8.67} & \underline{45.93} & \textbf{68.93} & \underline{2.34} & \cellcolor{green!50}\textbf{0.680} & 66.26 & 8.54 & 44.55 & 66.33 & 2.23 & 0.659 & \cellcolor{red!25}0.603 \\
 & \textsc{Freeze}  & 65.76 & 8.49 & 43.80 & 66.25 & 2.17 & \cellcolor{green!50}0.651 & 57.56 & 5.76 & 37.27 & 63.70 & 1.51 & 0.524 & \cellcolor{red!25}0.402 \\
 & \textsc{Stack}   & 67.47 & \underline{8.66} & 45.75 & 68.07 & 2.31 & \cellcolor{green!50}0.675 & \textbf{66.89} & \underline{8.58} & \underline{45.01} & \textbf{67.67} & \underline{2.26} & \underline{0.667} & \cellcolor{red!25}\underline{0.642} \\
 & \textsc{SVD}     & \underline{67.52} & 8.64 & \textbf{45.93} & \underline{68.51} & \textbf{2.35} & \cellcolor{green!50}\underline{0.678} & \underline{66.64} & \textbf{8.58} & \textbf{45.15} & \underline{67.40} & \textbf{2.28} & \textbf{0.667} & \cellcolor{red!25}\textbf{0.648} \\
\bottomrule
\end{tabular}
}
\end{sc}
\vspace{-0.1cm} % Adds a small gap 
  \caption*{ \scriptsize Amongst SFF shallow cut vs SFF deep cut vs FedB ($\equiv$ SFF with cut layer as last layer $\equiv$ all layers on client): \fcolorbox{black}{green!50}{\phantom{x}} Best \quad \fcolorbox{black}{red!25}{\phantom{x}} Worst }
\end{table*}

\begin{figure}[t]
    \centering
    \begin{subfigure}[t]{0.32\linewidth}
        \centering
        \includegraphics[width=\linewidth]{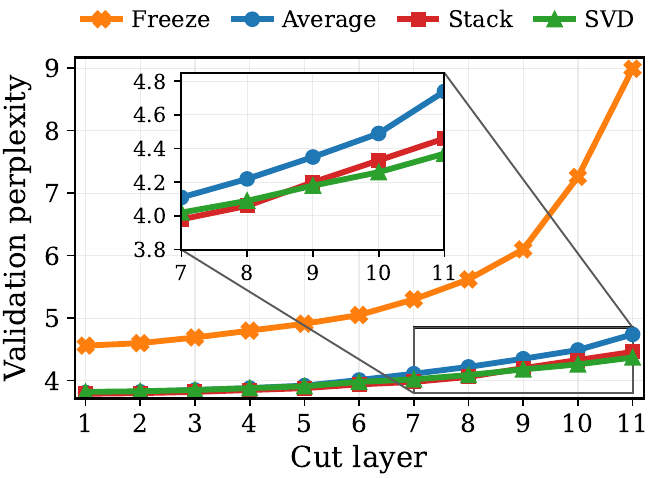}
        \caption{GPT-2 Small}
        \label{fig:ppl_vs_cutlayer_sm}
    \end{subfigure}
    \hfill
    \begin{subfigure}[t]{0.32\linewidth}
        \centering
        \includegraphics[width=\linewidth]{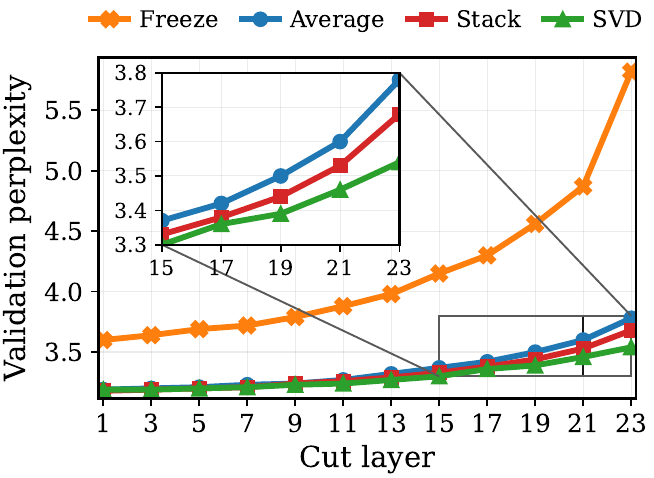}
        \caption{GPT-2 Medium}
        \label{fig:ppl_vs_cutlayer_md}
    \end{subfigure}
    \hfill
    \begin{subfigure}[t]{0.32\linewidth}
        \centering
        \includegraphics[width=\linewidth]{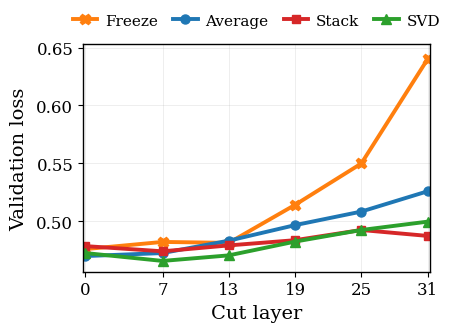}
        \caption{Llama-3-8B (GSM8K)}
        \label{fig:ppl_vs_cutlayer_llama}
    \end{subfigure}
    \caption{\textbf{Validation loss vs.\ cut-layer depth.} Loss (lower is
    better) as a function of $\ell_c$ for GPT-2 Small, Medium (heterogeneous
    LoRA ranks), and Llama-3-8B-Instruct on GSM8K (homogeneous, $r{=}16$).
    At shallow cuts, all methods perform similarly. At deep cuts, performance
    diverges sharply by aggregation method across all models and task types,
    with \textsc{Freeze} collapsing most severely and \textsc{SVD} retaining
    the strongest performance.}
    \label{fig:ppl_vs_cutlayer}
\end{figure}

%==========================================================================
\subsection{Aggregation Methods Diverge at Depth}
\label{sec:q2}
%==========================================================================

Section~\ref{sec:q1} establishes that aggregation choice becomes decisive at
deep cuts. Addressing \textbf{[RQ2]}: what distinguishes the methods? Each
makes a different structural trade-off when combining heterogeneous low-rank
adapters, producing qualitatively different noise characterized formally in Appendix~\ref{app:noise-bound}.
\textsc{Average} flattens the singular value spectrum: the low-rank task signal
contributes $O(1/K)$ to the aggregate while cross-client noise contributes
$O(1/\sqrt{K})$, a ratio that worsens with client count, yielding a high-rank
``white noise'' update that dilutes task-specific directions.
\textsc{Freeze} avoids cross-term interference by fixing a shared projection,
but constrains all clients to a single subspace, creating an irreducible
approximation bias for any task features orthogonal to it.
\textsc{Stack} preserves all client subspaces exactly, but interference noise also grows as $O(1/\sqrt{K})$.
\textsc{SVD}~\cite{bai2024federated} aggregates in update space and retains
only the dominant singular components, truncating to each client's local rank
$r_k$ to project out the high-frequency noise tail while preserving the
principal directions of cross-client consensus, explaining why it degrades
least at deep cuts.

The reason these noise differences matter more at deep cuts than shallow ones
is a structural property of the Transformer server. In CNN-based split
architectures, pooling layers progressively compress spatial resolution,
attenuating perturbations as they pass through successive layers. Transformers
have no analogous mechanism: residual connections and layer normalization are
designed to \emph{preserve} signal norms, not reduce them. This means that
whatever noise enters the server partition at the cut point exits essentially
intact, a claim we verify directly.

\begin{wraptable}{r}{0.46\textwidth}
\vspace{-6pt}
\centering
\small
\caption{\textbf{Noise propagation through the server partition.}
Final-layer perturbation ratio for the three deepest cut layers per GPT-2
model ($\pm$std, 64 inputs). Values $\geq 1.0$ confirm that noise injected
at the cut point exits the server partition unattenuated. A contractive
(CNN-like) architecture would produce values decaying toward 0.}
\label{tab:isometric_tube}
\resizebox{0.45\textwidth}{!}{
\begin{tabular}{lccc}
\toprule
\textbf{GPT-2} & $\ell_c = L{-}3$ & $\ell_c = L{-}2$ & $\ell_c = L{-}1$ \\
\midrule
S \scriptsize{$\{7,9,11\}$}   & $1.89{\pm}0.08$ & $1.55{\pm}0.05$ & $1.23{\pm}0.06$ \\
M \scriptsize{$\{14,18,23\}$} & $1.58{\pm}0.06$ & $1.37{\pm}0.02$ & $1.13{\pm}0.01$ \\
L \scriptsize{$\{21,28,35\}$} & $1.46{\pm}0.04$ & $1.22{\pm}0.02$ & $1.04{\pm}0.01$ \\
\bottomrule
\end{tabular}
}
\vspace{-4pt}
\end{wraptable}

We verify this via a perturbation propagation experiment. For each model and
cut layer $\ell_c$, we inject an L2-normalized Gaussian noise vector
$\boldsymbol{\varepsilon}$ at the cut-point hidden state and measure
$\|\tilde{x}_L - x_L\| / \|\boldsymbol{\varepsilon}\|$ at the final layer
after running the perturbed forward pass through the remaining server layers
on frozen pretrained weights (no LoRA, no training, no federation), averaged
over 64 random inputs.

Table~\ref{tab:isometric_tube} shows that across all models and cut depths,
the perturbation ratio never drops below 1.0. At $\ell_c = L{-}3$, GPT-2 Small
yields $1.89{\pm}0.08$; even at $\ell_c = L{-}1$ it remains $1.23{\pm}0.06$.
A contractive architecture would produce ratios decaying toward 0. Instead,
aggregation noise injected at the cut point exits the server partition
\emph{unattenuated or amplified}, propagating through the residual stream
without shrinking. This explains why aggregation noise that is absorbed
harmlessly at shallow cuts (where many layers follow the cut point) becomes
catastrophic at deep cuts (where the server partition is thin and provides no
compression). The transition from aggregation-invariant to aggregation-sensitive
behavior observed in Section~\ref{sec:q1} is therefore partly explained by the
server's loss of buffering capacity as depth increases.

This result also clarifies the FL baseline finding: in standard FL, aggregation
noise is injected across the \emph{entire} model with no partition to provide
even minimal buffering. Every layer processes noisy updates from the first
forward pass, and there are no subsequent layers to compensate. SFF's advantage
over FL, even at deep cuts, follows directly from the fact that at least some
unperturbed server processing occurs before the prediction head, a buffer that
vanishes only in the limit $\ell_c \to L$.

\begin{tcolorbox}[takeaway]
\textbf{Key Takeaway:} Unlike CNNs, the Transformer server partition does not
attenuate aggregation noise, it propagates it intact from cut point to output.
This near-isometric propagation explains why noise that is harmless at shallow
cuts, where the server provides a long buffer, becomes catastrophic at deep cuts
where the buffer is thin. It also explains why all SFF configurations outperform
FL: even the deepest cut retains some unperturbed server processing.
\end{tcolorbox}

%==========================================================================
\subsection{Structural Failure in Deep Server Layers}
\label{sec:q3}
%==========================================================================

Addressing \textbf{[RQ3]}: the propagation result shows that noise is not
attenuated, but does not fully account for the sharp \emph{acceleration} in
performance degradation visible in Figure~\ref{fig:ppl_vs_cutlayer} beyond a
specific depth. If propagation were uniformly near-isometric at all depths,
degradation would scale smoothly. Instead, we observe a characteristic
inflection that aligns with a discrete structural change in the server's
processing capacity.

\textbf{Attention Collapse.}
Following~\cite{sanyal2025attentioncollapsesdegeneratelayers}, we measure the
\emph{effective rank} of attention matrices at each layer by computing the
singular value spectrum and applying a 90\% spectral energy threshold. Figure~\ref{fig:rank_collapse} shows results for GPT-2 Small, Medium, and
Llama-3-8B. Effective rank is high in early layers and drops sharply at a
characteristic depth: $\ell \approx 6$ for GPT-2 Small, $\ell \approx 11$
for Medium, and $\ell \approx 19$ for LLaMA. Beyond this point, attention
matrices degenerate into near-identity, low-rank mappings, \textbf{Attention
Collapse}. These thresholds align precisely with the depths at which performance
degradation accelerates in Figure~\ref{fig:ppl_vs_cutlayer} and
Figure~\ref{fig:llama_acc}: for LLaMA, the accuracy decline across all
aggregation methods begins at $\ell_c \approx 19$, exactly matching the
measured collapse threshold. Prior to collapse, high-rank server layers can
project noisy activations toward the data manifold; after collapse, the server
tail consists of near-identity layers with no corrective capacity, and noise
propagating through the residual stream exits directly into the prediction head.

\begin{figure}[t]
    \centering
    \begin{subfigure}[t]{0.32\linewidth}
        \centering
        \includegraphics[width=0.9\linewidth]{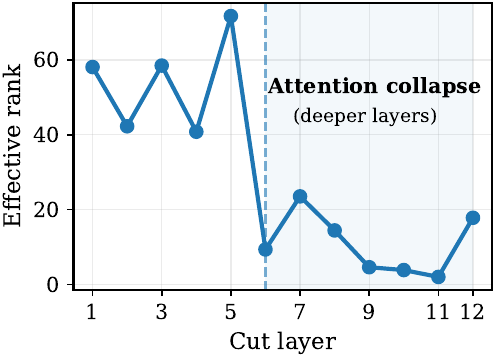}
        \caption{GPT-2 S ($\ell \approx 6$)}
        \label{fig:rank_collapse_sm}
    \end{subfigure}
    \hfill
    \begin{subfigure}[t]{0.32\linewidth}
        \centering
        \includegraphics[width=0.9\linewidth]{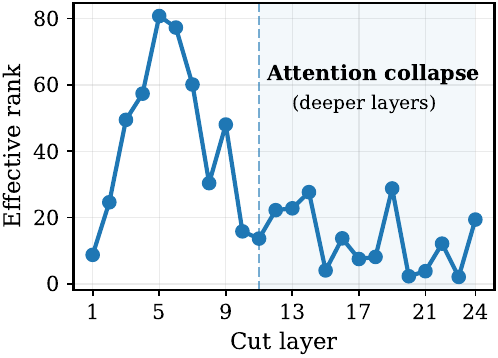}
        \caption{GPT-2 M ($\ell \approx 11$)}
        \label{fig:rank_collapse_md}
    \end{subfigure}
    \hfill
    \begin{subfigure}[t]{0.32\linewidth}
        \centering
        \includegraphics[width=0.9\linewidth]{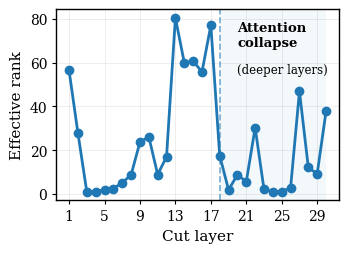}
        \caption{Llama-3-8B ($\ell \approx 19$)}
        \label{fig:rank_collapse_llama}
    \end{subfigure}
    \caption{\textbf{Attention collapse across depth.} Maximum effective rank
    (90\% spectral energy threshold) of attention matrices per layer for
    GPT-2 Small, Medium, and Llama-3-8B-Instruct. Rank drops sharply after
    $\ell \approx \{6, 11, 19\}$ respectively, coinciding with the depths at
    which performance degradation accelerates in Figures~\ref{fig:ppl_vs_cutlayer}
    and~\ref{fig:llama_acc}. The collapse threshold scales with model depth,
    consistently occurring at approximately 50--60\% of total layer depth.}
    \vspace{-2em}\label{fig:rank_collapse}
\end{figure}

\paragraph{Sequential server processing.}
Attention collapse compounds a second structural property of SFF: the Main
Server processes client activations one at a time, so update steps for client
$k$ shift the shared server state before client $j$ is processed. At shallow
cuts, high-rank server layers provide geometric freedom for diverse client
gradients to occupy non-conflicting subspaces. At deep cuts, attention collapse
reduces the server's effective dimensionality: when task diversity exceeds
geometric capacity, sequential updates inevitably collide, and each client
partially overwrites the progress of the others. The residual performance gap
that persists under \textsc{SVD}, which addresses aggregation noise upstream
but not server-side dynamics, is consistent with this compounding factor.

\begin{tcolorbox}[takeaway]
\textbf{Key Takeaway:} Deep server layers degenerate via Attention Collapse,
eliminating the nonlinear capacity needed to correct incoming noise. This
explains the characteristic performance cliff: before the collapse threshold,
the server can partially compensate for aggregation noise; after it, noise
arrives intact and exits uncorrected. The collapse threshold scales
consistently with model depth, occurring at ${\approx}50$--$60\%$ of total
layers across GPT-2 and LLaMA. Sequential server processing further narrows
the geometric space available for client updates to coexist.
\end{tcolorbox}

Our findings yield a coherent account of the Depth-Performance Dilemma.
Aggregation noise is the dominant controllable factor: methods that preserve
the low-rank update manifold (\textsc{SVD}, \textsc{Stack}) degrade
substantially less at deep cuts, and all SFF configurations outperform
standard FL by retaining at least some unperturbed server processing. Near-isometric propagation ensures that noise persists throughout the partition, while attention collapse removes the nonlinear capacity required to recover.
%Near-isometric propagation explains why noise cannot be recovered once it enters the partition, and attention collapse eliminates the corrective capacity that would otherwise compensate.
\section{Conclusion and Future Work}
\label{sec:conclusion}
 
We conducted the first systematic empirical investigation of cut-layer depth
in Split-Federated Fine-tuning of LLMs, uncovering a fundamental
\textbf{Depth-Performance Dilemma}: the partition depths that maximize
throughput and privacy are precisely where fine-tuning quality collapses,
a finding that holds across GPT-2 Small/Medium/Large and Llama-3-8B-Instruct,
on E2E NLG, GSM8K, and GLUE, under every aggregation method studied. Through
controlled experiments, we traced the collapse to three compounding structural
factors: aggregation noise whose severity is governed by how well each method
preserves the low-rank update manifold; near-isometric noise propagation
through the Transformer's residual architecture, which delivers aggregation
artifacts intact to the output; and attention collapse in deep server layers,
which eliminates corrective capacity precisely where noise arrives most intact,
further compounded by sequential server processing. These results challenge the
assumption, implicit in all prior SFF work, that partition depth is a
utility-neutral efficiency knob, depth interacts with Transformer topology
and aggregation design in ways that determine whether fine-tuning succeeds or
collapses. Two concrete directions follow from this diagnosis. First, server-side
execution strategies that accumulate gradients across all clients before
updating, approximating centralized optimization, would eliminate the
moving-target problem and recover the geometric separation that attention
collapse currently destroys. Second, the aggregation design space remains open:
even in standard FL, federated LoRA methods suffer from subspace
misalignment under non-IID data~\cite{zhang2026fedrotlora}, and in SFF this
misaligned noise enters a near-isometric partition immediately, amplifying its
consequences. Aggregation schemes that explicitly account for subspace alignment
and cut-layer topology represent a principled path toward resolving the
Depth-Performance Dilemma. %We release our full experimental pipeline to support exploration of these directions: \url{https://anonymous.4open.science/r/SFF-C1FD/README.md}.

% \nocite pulled a stray template reference ("langley00") into the
% bibliography; disabled for the preprint. Re-enable only if you actually
% want an uncited entry to appear.
% \nocite{langley00}

\bibliographystyle{unsrt} % references listed in order of appearance starting from [1]
\bibliography{references}

@article{DBLP:journals/corr/abs-1812-00564,
  author       = {Praneeth Vepakomma and
                  Otkrist Gupta and
                  Tristan Swedish and
                  Ramesh Raskar},
  title        = {Split learning for health: Distributed deep learning without sharing
                  raw patient data},
  journal      = {CoRR},
  volume       = {abs/1812.00564},
  year         = {2018},
  url          = {http://arxiv.org/abs/1812.00564},
  eprinttype    = {arXiv},
  eprint       = {1812.00564},
  bibsource    = {dblp computer science bibliography, https://dblp.org}
}

@InProceedings{pmlr-v54-mcmahan17a,
  title = 	 {{Communication-Efficient Learning of Deep Networks from Decentralized Data}},
  author = 	 {McMahan, Brendan and Moore, Eider and Ramage, Daniel and Hampson, Seth and Arcas, Blaise Aguera y},
  booktitle = 	 {Proceedings of the 20th International Conference on Artificial Intelligence and Statistics},
  pages = 	 {1273--1282},
  year = 	 {2017},
  editor = 	 {Singh, Aarti and Zhu, Jerry},
  volume = 	 {54},
  series = 	 {Proceedings of Machine Learning Research},
  month = 	 {20--22 Apr},
  publisher =    {PMLR},
  url = 	 {https://proceedings.mlr.press/v54/mcmahan17a.html}
}

@inproceedings{DBLP:journals/corr/abs-2106-09685,
  author       = {Edward J. Hu and
                  Yelong Shen and
                  Phillip Wallis and
                  Zeyuan Allen{-}Zhu and
                  Yuanzhi Li and
                  Shean Wang and
                  Lu Wang and
                  Weizhu Chen},
  title        = {{LoRA}: Low-Rank Adaptation of Large Language Models},
  booktitle    = {The Tenth International Conference on Learning Representations, {ICLR} 2022},
  year         = {2022},
  publisher    = {OpenReview.net},
  url          = {https://openreview.net/forum?id=nZeVKeeFYf9},
  eprint       = {2106.09685},
  archivePrefix = {arXiv},
  primaryClass = {cs.CL}
}

@article{Thapa_Mahawaga_Arachchige_Camtepe_Sun_2022, title={SplitFed: When Federated Learning Meets Split Learning}, volume={36}, url={https://ojs.aaai.org/index.php/AAAI/article/view/20825}, DOI={10.1609/aaai.v36i8.20825}, abstractNote={Federated learning (FL) and split learning (SL) are two popular distributed machine learning approaches. Both follow a model-to-data scenario; clients train and test machine learning models without sharing raw data. SL provides better model privacy than FL due to the machine learning model architecture split between clients and the server. Moreover, the split model makes SL a better option for resource-constrained environments. However, SL performs slower than FL due to the relay-based training across multiple clients. In this regard, this paper presents a novel approach, named splitfed learning (SFL), that amalgamates the two approaches eliminating their inherent drawbacks, along with a refined architectural configuration incorporating differential privacy and PixelDP to enhance data privacy and model robustness. Our analysis and empirical results demonstrate that (pure) SFL provides similar test accuracy and communication efficiency as SL while significantly decreasing its computation time per global epoch than in SL for multiple clients. Furthermore, as in SL, its communication efficiency over FL improves with the number of clients. Besides, the performance of SFL with privacy and robustness measures is further evaluated under extended experimental settings.}, number={8}, journal={Proceedings of the AAAI Conference on Artificial Intelligence}, author={Thapa, Chandra and Mahawaga Arachchige, Pathum Chamikara and Camtepe, Seyit and Sun, Lichao}, year={2022}, month={Jun.}, pages={8485-8493} }

@inproceedings{10.5555/3737916.3741203,
author = {Han, Pengchao and Huang, Chao and Tian, Geng and Tang, Ming and Liu, Xin},
title = {Convergence analysis of split federated learning on heterogeneous data},
year = {2024},
isbn = {9798331314385},
publisher = {Curran Associates Inc.},
address = {Red Hook, NY, USA},
booktitle = {Proceedings of the 38th International Conference on Neural Information Processing Systems},
articleno = {3287},
numpages = {69},
location = {Vancouver, BC, Canada},
series = {NIPS '24}
}

@misc{dachille2024impactcutlayerselection,
      title={The Impact of Cut Layer Selection in Split Federated Learning}, 
      author={Justin Dachille and Chao Huang and Xin Liu},
      year={2024},
      eprint={2412.15536},
      archivePrefix={arXiv},
      primaryClass={cs.DC},
      url={https://arxiv.org/abs/2412.15536}, 
}

@misc{lin2024splitlorasplitparameterefficientfinetuning,
      title={SplitLoRA: A Split Parameter-Efficient Fine-Tuning Framework for Large Language Models}, 
      author={Zheng Lin and Xuanjie Hu and Yuxin Zhang and Zhe Chen and Zihan Fang and Xianhao Chen and Ang Li and Praneeth Vepakomma and Yue Gao},
      year={2024},
      eprint={2407.00952},
      archivePrefix={arXiv},
      primaryClass={cs.LG},
      url={https://arxiv.org/abs/2407.00952}, 
}

@article{lin2025hsplitloraheterogeneoussplitparameterefficient,
  title={HSplitLoRA: A Heterogeneous Split Parameter-Efficient Fine-Tuning Framework for Large Language Models},
  author={Zheng Lin and Yu-xin Zhang and Zhe Chen and Zihan Fang and Xianhao Chen and Praneeth Vepakomma and Wei Ni and Jun Luo and Yue Gao},
  journal={ArXiv},
  year={2025},
  volume={abs/2505.02795},
  url={https://api.semanticscholar.org/CorpusID:278339039}
}

@article{10.1145/3065386,
author = {Krizhevsky, Alex and Sutskever, Ilya and Hinton, Geoffrey E.},
title = {ImageNet classification with deep convolutional neural networks},
year = {2017},
issue_date = {June 2017},
publisher = {Association for Computing Machinery},
address = {New York, NY, USA},
volume = {60},
number = {6},
issn = {0001-0782},
url = {https://doi.org/10.1145/3065386},
doi = {10.1145/3065386},
journal = {Commun. ACM},
month = may,
pages = {84–90},
numpages = {7}
}

@misc{ma2025splitfrozensplitlearningdeviceside,
      title={SplitFrozen: Split Learning with Device-side Model Frozen for Fine-Tuning LLM on Heterogeneous Resource-Constrained Devices}, 
      author={Jian Ma and Xinchen Lyu and Jun Jiang and Qimei Cui and Haipeng Yao and Xiaofeng Tao},
      year={2025},
      eprint={2503.18986},
      archivePrefix={arXiv},
      primaryClass={cs.LG},
      url={https://arxiv.org/abs/2503.18986}, 
}

@misc{wang2023privateloraefficientprivacypreserving,
      title={PrivateLoRA For Efficient Privacy Preserving LLM}, 
      author={Yiming Wang and Yu Lin and Xiaodong Zeng and Guannan Zhang},
      year={2023},
      eprint={2311.14030},
      archivePrefix={arXiv},
      primaryClass={cs.AI},
      url={https://arxiv.org/abs/2311.14030}, 
}

@inproceedings{
nair2025fslsage,
title={{FSL}-{SAGE}: Accelerating Federated Split Learning via Smashed Activation Gradient Estimation},
author={Srijith Nair and Michael Lin and Peizhong Ju and Amirreza Talebi and Elizabeth Serena Bentley and Jia Liu},
booktitle={Forty-second International Conference on Machine Learning},
year={2025},
url={https://openreview.net/forum?id=HnwcrtoDd4}
}

@inproceedings{
zhang2023towards,
title={Towards Building the Federated{GPT}: Federated Instruction Tuning},
author={Jianyi Zhang and Saeed Vahidian and Martin Kuo and Chunyuan Li and Ruiyi Zhang and Tong Yu and Guoyin Wang and Yiran Chen},
booktitle={International Workshop on Federated Learning in the Age of Foundation Models in Conjunction with NeurIPS 2023},
year={2023},
url={https://openreview.net/forum?id=TaDiklyVps}
}

@inproceedings{
sun2024improving,
title={Improving Lo{RA} in Privacy-preserving Federated Learning},
author={Youbang Sun and Zitao Li and Yaliang Li and Bolin Ding},
booktitle={The Twelfth International Conference on Learning Representations},
year={2024},
url={https://openreview.net/forum?id=NLPzL6HWNl}
}

@inproceedings{
bai2024federated,
title={Federated Fine-tuning of Large Language Models under Heterogeneous Tasks and Client Resources},
author={Jiamu Bai and Daoyuan Chen and Bingchen Qian and Liuyi Yao and Yaliang Li},
booktitle={The Thirty-eighth Annual Conference on Neural Information Processing Systems},
year={2024},
url={https://openreview.net/forum?id=gkOzoHBXUw}
}

@inproceedings{flora,
author = {He, Yexiao and Li, Ang and Lyu, Lingjuan and Shen, Zheyu and Sun, Guoheng and Wang, Hongyi and Wang, Ziyao},
year = {2024},
month = {01},
pages = {22513-22533},
title = {FLoRA: Federated Fine-Tuning Large Language Models with Heterogeneous Low-Rank Adaptations},
doi = {10.52202/079017-0708}
}

@inproceedings{cho-etal-2024-heterogeneous,
    title = "Heterogeneous {L}o{RA} for Federated Fine-tuning of On-Device Foundation Models",
    author = "Cho, Yae Jee  and
      Liu, Luyang  and
      Xu, Zheng  and
      Fahrezi, Aldi  and
      Joshi, Gauri",
    editor = "Al-Onaizan, Yaser  and
      Bansal, Mohit  and
      Chen, Yun-Nung",
    booktitle = "Proceedings of the 2024 Conference on Empirical Methods in Natural Language Processing",
    month = nov,
    year = "2024",
    address = "Miami, Florida, USA",
    publisher = "Association for Computational Linguistics",
    url = "https://aclanthology.org/2024.emnlp-main.717/",
    doi = "10.18653/v1/2024.emnlp-main.717",
    pages = "12903--12913"
}

@inproceedings{novikova-etal-2017-e2e,
    title = "The {E}2{E} Dataset: New Challenges For End-to-End Generation",
    author = "Novikova, Jekaterina  and
      Du{\v{s}}ek, Ond{\v{r}}ej  and
      Rieser, Verena",
    editor = "Jokinen, Kristiina  and
      Stede, Manfred  and
      DeVault, David  and
      Louis, Annie",
    booktitle = "Proceedings of the 18th Annual {SIG}dial Meeting on Discourse and Dialogue",
    month = aug,
    year = "2017",
    address = {Saarbr{\"u}cken, Germany},
    publisher = "Association for Computational Linguistics",
    url = "https://aclanthology.org/W17-5525/",
    doi = "10.18653/v1/W17-5525",
    pages = "201--206"
}

@article{radford2019language,
  title={Language Models are Unsupervised Multitask Learners},
  author={Radford, Alec and Wu, Jeff and Child, Rewon and Luan, David and Amodei, Dario and Sutskever, Ilya},
  year={2019}
}

@inproceedings{Chen_2024, series={CCS ’24},
   title={Unveiling the Vulnerability of Private Fine-Tuning in Split-Based Frameworks for Large Language Models: A Bidirectionally Enhanced Attack},
   url={http://dx.doi.org/10.1145/3658644.3690295},
   DOI={10.1145/3658644.3690295},
   booktitle={Proceedings of the 2024 on ACM SIGSAC Conference on Computer and Communications Security},
   publisher={ACM},
   author={Chen, Guanzhong and Qin, Zhenghan and Yang, Mingxin and Zhou, Yajie and Fan, Tao and Du, Tianyu and Xu, Zenglin},
   year={2024},
   month=dec, pages={2904–2918},
   collection={CCS ’24} }

@inproceedings{10.24963/ijcai.2025/57,
author = {Shen, Xicong and Liu, Yang and Liu, Yi and Wang, Peiran and Liu, Huiqi and Hong, Jue and Duan, Bing and Huang, Zirui and Mao, Yunlong and Wu, Ye and Zhong, Sheng},
title = {SAP: privacy-preserving fine-tuning on language models with split-and-privatize framework},
year = {2025},
isbn = {978-1-956792-06-5},
url = {https://doi.org/10.24963/ijcai.2025/57},
doi = {10.24963/ijcai.2025/57},
booktitle = {Proceedings of the Thirty-Fourth International Joint Conference on Artificial Intelligence},
articleno = {57},
numpages = {9},
location = {Montreal, Canada},
series = {IJCAI '25}
}

@article{sanyal2025attentioncollapsesdegeneratelayers,
  title={When Attention Collapses: How Degenerate Layers in LLMs Enable Smaller, Stronger Models},
  author={Sunny Sanyal and Ravid Shwartz-Ziv and Alexandros G. Dimakis and Sujay Sanghavi},
  journal={Trans. Mach. Learn. Res.},
  year={2024},
  volume={2026},
  url={https://api.semanticscholar.org/CorpusID:269137423}
}

@inproceedings{ramesh2025floristsingularvaluethresholding,
 author = {Ramesh, Hariharan and Dass, Jyotikrishna},
 booktitle = {Proceedings of Machine Learning and Systems},
 editor = {A. Chowdhery and Z. Jia},
 pages = {1--21},
 publisher = {MLSys},
 title = {FLoRIST: Singular Value Thresholding for Efficient and Accurate Federated Fine-Tuning of Large Language Models},
 url = {https://proceedings.mlsys.org/paper_files/paper/2026/file/f38cb4cf9a5eaa92b3cfa481832719c6-Paper-Conference.pdf},
 volume = {8},
 year = {2026}
}

@article{zhang2026fedrotlora,
  title   = {{FedRot-LoRA}: Mitigating Rotational Misalignment in Federated {LoRA}},
  author  = {Zhang, Haoran and Kim, Dongjun and Cha, Seohyeon and Vikalo, Haris},
  journal = {arXiv preprint arXiv:2602.23638},
  year    = {2026}
}

@article{llama3,
  title  = {The Llama 3 Herd of Models},
  author = {Meta AI},
  journal= {arXiv preprint arXiv:2407.21783},
  year   = {2024}
}

@article{gsm8k,
  title  = {Training Verifiers to Solve Math Word Problems},
  author = {Cobbe, Karl and others},
  journal= {arXiv preprint arXiv:2110.14168},
  year   = {2021}
}

@inproceedings{DBLP:journals/corr/abs-1804-07461,
    title = "{GLUE}: A Multi-Task Benchmark and Analysis Platform for Natural Language Understanding",
    author = "Wang, Alex  and
      Singh, Amanpreet  and
      Michael, Julian  and
      Hill, Felix  and
      Levy, Omer  and
      Bowman, Samuel R.",
    editor = "Linzen, Tal  and
      Chrupa{\l}a, Grzegorz  and
      Alishahi, Afra",
    booktitle = "Proceedings of the 2018 {EMNLP} Workshop {B}lackbox{NLP}: Analyzing and Interpreting Neural Networks for {NLP}",
    month = nov,
    year = "2018",
    address = "Brussels, Belgium",
    publisher = "Association for Computational Linguistics",
    url = "https://aclanthology.org/W18-5446/",
    doi = "10.18653/v1/W18-5446",
    pages = "353--355"
}
%%%%%%%%%%%%%%%%%%%%%%%%%%%%%%%%%%%%%%%%%%%%%%%%%%%%%%%%%%%%
\appendix
\newpage
\section{Extended Experimental Details}
\label{app:exp_details}

\begin{table}[h]
\vspace{-4pt}
\centering
\small
\caption{\textbf{Overview of experimental configurations.} All setups use
non-IID data partitioned via Dirichlet ($\alpha{=}0.5$) and exhaust every
cut-layer position (or a uniform stride for LLaMA).}
\label{tab:exp_overview}
\setlength{\tabcolsep}{4pt}
\begin{tabular}{llllll}
\toprule
\textbf{Model} & \textbf{Benchmark} & \textbf{Clients} & \textbf{Sampled} & \textbf{Rounds} & \textbf{LoRA rank} \\
\midrule
GPT-2 S/M/L   & E2E NLG & 3  & 3  & 100  & $\{4,8,16\}$ (heter.) \\
GPT-2 S/M/L   & E2E NLG & 3  & 3  & 100  & 16 (homo.) \\
GPT-2 Small   & GLUE (MNLI, QQP) & 100 & 10 & 200 & 16 (homo.) \\
Llama-3-8B    & GSM8K   & 30 & 3  & 50   & 16 (homo.) \\
\bottomrule
\end{tabular}
\vspace{-6pt}
\end{table}

\subsection{GPT-2 on E2E NLG}
\label{app:gpt2_e2e}
 
\paragraph{Dataset.}
The E2E NLG benchmark~\cite{novikova-etal-2017-e2e} is a conditional generation
dataset in the restaurant domain containing approximately 42K training, 4.6K
validation, and 4.6K test samples. Each sample consists of a structured meaning
representation and one or more reference utterances. We evaluate using five
automatic metrics: BLEU, NIST, METEOR, ROUGE-L, and CIDEr.
 
\paragraph{Models and training.}
We evaluate GPT-2 Small (12 layers, 768 hidden, 124M parameters), Medium
(24 layers, 1024 hidden, 355M), and Large (36 layers, 1280 hidden, 774M)~\cite{radford2019language}.
All models are fine-tuned with batch size 8, learning rate $2\times10^{-4}$,
and maximum sequence length 512, for 1 epoch (GPT-2 Small) or 2 epochs (Medium
and Large). LoRA adapters ($\Delta\mathbf{W} = \mathbf{B}\mathbf{A}$,
$\mathbf{B}\in\mathbb{R}^{d\times r}$, $\mathbf{A}\in\mathbb{R}^{r\times d}$)
are applied to the query and value projection matrices of every attention block.
 
\paragraph{Federated setup — heterogeneous (primary).}
$N{=}3$ clients, all participating every round. Client datasets are partitioned
non-IID via Dirichlet ($\alpha{=}0.5$), yielding skewed label distributions.
LoRA ranks are assigned heterogeneously: $r_i \in \{4, 8, 16\}$ uniformly at
random per client. Client adapters are aggregated every $I{=}100$ local
optimization steps. Methods requiring aligned LoRA factors apply zero-padding
rank alignment following HetLoRA~\cite{cho-etal-2024-heterogeneous}.
 
\paragraph{Cut-layer sweep.}
The cut layer $\ell_c$ is swept exhaustively over $\{1, \dots, L{-}1\}$ for
each model and aggregation method, yielding $4\times(L{-}1)$ independent SFF
runs per model (48 for Small, 92 for Medium, 140 for Large). Under rank
heterogeneity, all methods except \textsc{Stack} produce client-specific global
adapters, requiring per-client evaluation and averaging, further multiplying
experimental cost.
 
\subsection{GPT-2 Small on GLUE}
\label{app:glue}
 
\paragraph{Dataset.}
We evaluate on two GLUE tasks: MNLI (Multi-Genre Natural
Language Inference, 393K train / 9.8K val samples) and QQP (Quora Question
Pairs, 364K train / 40K val samples). Both are cast as classification tasks.
 
\paragraph{Model and training.}
GPT-2 Small (12 layers, 124M) is fine-tuned with batch size 8, learning rate
$2\times10^{-4}$, and maximum sequence length 512. LoRA adapters ($r{=}16$) are
applied to query and value projections of every attention block.
 
\paragraph{Federated setup.}
100 total clients, 10 sampled per round, for 200 communication rounds. Client
datasets are partitioned non-IID via Dirichlet ($\alpha{=}0.5$). All clients
use homogeneous rank $r{=}16$. Client adapters are aggregated every $I{=}100$
local steps. This large-client configuration tests the scalability of the
Depth-Performance Dilemma and aggregation sensitivity beyond the 3-client primary
setup.
 
\paragraph{Cut-layer sweep.}
Same exhaustive sweep as the E2E experiments: $\ell_c \in \{1,\dots,11\}$,
giving $4\times11 = 44$ independent SFF runs per GLUE task.
 
\paragraph{Results.}
The Depth-Performance Dilemma holds on both GLUE tasks: classification accuracy
degrades consistently with cut-layer depth across all four aggregation methods,
with the same aggregation ordering (\textsc{SVD} and \textsc{Stack} outperforming
\textsc{Average} and \textsc{Freeze} at deep cuts) observed in the E2E primary
experiments. The 100-client setup amplifies aggregation-method differences at
depth relative to the 3-client setup, consistent with the $O(\sqrt{K})$
cross-task interference scaling of \textsc{Stack} and the $O(1/\sqrt{K})$ signal
dilution of \textsc{Average}. Table~\ref{tab:glue}
report accuracies at shallow and deep cut for MNLI and QQP respectively.

 \begin{table}[h]
\centering
\caption{GPT-2 Small results on GLUE benchmark (QQP and MNLI) with 100 clients (10 sampled per round) under homogeneous LoRA rank ($r=16$). Shallow (resp. deep) values are obtained by averaging over the first (resp. last) 25\% of possible cut layers. Best results per task and regime in \textbf{bold}; second-best \underline{underlined}.}
\label{tab:glue}
\begin{tabular}{llcc}
\toprule
\textbf{Task} & \textbf{Method} & \textbf{Shallow (\%)} & \textbf{Deep (\%)} \\
\midrule
\multirow{4}{*}{QQP}
 & \textsc{Average} & \cellcolor{green!50}86.35 & \underline{83.60} \\
 & \textsc{Freeze}  & \cellcolor{green!50}84.66 & 79.34 \\
 & \textsc{Stack}   & \cellcolor{green!50}\underline{84.24} & 80.41 \\
 & \textsc{SVD}     & \cellcolor{green!50}\textbf{86.26} & \textbf{83.63} \\
\midrule
\multirow{4}{*}{MNLI}
 & \textsc{Average} & \cellcolor{green!50}\underline{77.44} & \underline{73.37} \\
 & \textsc{Freeze}  & \cellcolor{green!50}73.87 & 63.37 \\
 & \textsc{Stack}   & \cellcolor{green!50}73.50 & 66.51 \\
 & \textsc{SVD}     & \cellcolor{green!50}\textbf{77.55} & \textbf{73.39} \\
\bottomrule
\end{tabular}
\end{table}

\subsection{Llama-3-8B-Instruct on GSM8K}
\label{app:llama_gsm8k}
 
\paragraph{Dataset.}
GSM8K~\cite{gsm8k} is a benchmark of 8.5K grade-school mathematics word problems
(7.5K train / 1K test) requiring multi-step arithmetic reasoning. We report
exact-match accuracy on the test set.
 
\paragraph{Model and training.}
Llama-3-8B-Instruct~\cite{llama3} (32 Transformer layers, 4096 hidden dimension,
8B parameters) is fine-tuned with batch size 1, learning rate $1\times10^{-4}$,
maximum sequence length 1024, and 1 local epoch per communication round. LoRA
adapters ($r{=}16$) are applied to the key and value projection matrices of
every attention block.
 
\paragraph{Federated setup.}
30 total clients, 3 sampled per round, for 50 communication rounds. Client
datasets are partitioned non-IID via Dirichlet ($\alpha{=}0.5$). All clients
use homogeneous rank $r{=}16$.
 
\paragraph{Cut-layer sweep.}
To manage the computational cost of exhaustive sweeping on a 32-layer, 8B
model, cut layers are evaluated at stride 3: $\ell_c \in \{1, 4, 7, 10, 13,
16, 19, 22, 25, 28, 31\}$, giving $4\times11 = 44$ independent SFF runs. This
stride is fine enough to capture the characteristic depth threshold at which
performance acceleration occurs (observed around $\ell_c \approx 20$ for a
32-layer model, consistent with the attention collapse profile at roughly 60\%
depth), while remaining computationally tractable.
 
\paragraph{Results.}
The Depth-Performance Dilemma holds at 8B scale on mathematical reasoning: GSM8K
accuracy degrades consistently with cut-layer depth across all four aggregation
methods, with the same aggregation ordering observed in the GPT-2 experiments.
The attention collapse threshold, measured by effective rank of attention
matrices, shifts to $\ell \approx 20$ for LLaMA, consistent with the
${\approx}60\%$ depth pattern observed across the GPT-2 family. These results
confirm that the dilemma is not an artifact of GPT-2's architecture or the E2E
NLG task, but a general property of Transformer-based SFF. Figure~\ref{fig:llama_acc}
reports accuracy vs.\ cut-layer depth for all four aggregation methods. 

\paragraph{Hardware setup and cost.}
All experiments are run on an HPC cluster equipped with NVIDIA H200 GPUs.
Each individual SFF run for LLaMA, one cut-layer configuration under one aggregation
method for 50 communication rounds, requires approximately 9 hours of
wall-clock time. The full LLaMA sweep ($4 \times 11 = 44$ configurations)
therefore represents roughly 396 GPU-hours of computation, underscoring the
practical cost of exhaustive depth evaluation at 8B scale and motivating the
stride-3 cut-layer sampling described above.

% \subsection{Llama-3-8B-Instruct Results}
% \label{app:llama_extras}

\newpage
\section{Pseudocode}
\label{app:pseudocode}
\begin{algorithm}[h]
\caption{Split-Federated Fine-tuning (SFF); Aligned with Figure 2}
\label{alg:sff}
\KwIn{Pretrained LLM with $L$ Transformer blocks $\{T_l\}_{l=1}^L$, cut layer $l_c$, $K$ clients with private datasets $\{\mathcal{D}_k\}_{k=1}^K$, client LoRA ranks $\{r_k\}_{k=1}^K$, aggregation interval $I$ steps, learning rates $\gamma_C$ (client), $\gamma_S$ (server)}
\KwOut{Fine-tuned model}

\textbf{Partition:} Client $k$ holds blocks $T_1, \ldots, T_{l_c}$ with LoRA adapters $(B_k \in \mathbb{R}^{d \times r_k}, A_k \in \mathbb{R}^{r_k \times d})$\;
Main Server holds blocks $T_{l_c+1}, \ldots, T_L$ with shared parameters $\Theta_S$\;

\textbf{Initialize:} Global LoRA adapters $(B^{(g)}, A^{(g)})$\;
\For{each global round $t = 1, 2, \ldots$}{

    Each client $k$ downloads global adapters: $(B_k, A_k) \leftarrow (B_k^{(g)}, A_k^{(g)})$\; \tcc*[f]{Step \circled{1}}

    \For{each local step $s = 1, \ldots, I$}{

        \tcc{\textbf{Step I:} Client-side Forward Pass (parallel)}
        \For{each client $k = 1, \ldots, K$ \textbf{in parallel}}{
            Sample minibatch $(\mathbf{x}_k, \mathbf{y}_k) \sim \mathcal{D}_k$\;
            $\mathbf{h}_{l_c,k} \leftarrow T_{l_c} \circ \cdots \circ T_1(\text{Embed}(\mathbf{x}_k))$\; \tcc*[f]{Step \circled{2}, using local LoRA $(B_k, A_k)$}
            
            Send smashed activations $\mathbf{h}_{l_c,k}$ to Main Server\; \tcc*[f]{Step \circled{3}}
        }

        \tcc{\textbf{Step II:} Main Server Forward + Backward (sequential)}
        \For{each client $k = 1, \ldots, K$ \textbf{sequentially}}{
            $\hat{\mathbf{y}}_k \leftarrow T_L \circ \cdots \circ T_{l_c+1}(\mathbf{h}_{l_c,k})$\; \tcc*[f]{Step \circled{4}}

            Compute loss $\mathcal{L}_k = \mathcal{L}(\hat{\mathbf{y}}_k, \mathbf{y}_k)$\; \tcc*[f]{Step \circled{5}}
            
            Backpropagate through $\Theta_S$: compute $\nabla_{\Theta_S} \mathcal{L}_k$\; \tcc*[f]{Step \circled{6}}
            
            Update: $\Theta_S \leftarrow \Theta_S - \gamma_S \nabla_{\Theta_S} \mathcal{L}_k$\; \tcc*[f]{Step \circled{9}}
            
            Send gradient $\nabla_{\mathbf{h}_{l_c,k}} \mathcal{L}_k$ to client $k$\; \tcc*[f]{Step \circled{7}}
        }

        \tcc{\textbf{Step III:} Client-side Backward Pass (parallel)}
        \For{each client $k = 1, \ldots, K$ \textbf{in parallel}}{
            $(B_k, A_k) \leftarrow (B_k, A_k) - \gamma_C \nabla_{(B_k, A_k)} \mathcal{L}_k$\; \tcc*[f]{Step \circled{8}}
        }
    }

    \tcc{\textbf{Step IV:} Federated Aggregation on Fed Server (every $I$ steps)}
    Clients upload local LoRA $(B_k, A_k)$ to Fed Server\; \tcc*[f]{Step \circled{12}}
    
    \tcc{Any aggregation method can be used: Avg, Freeze, Stack, SVD}
    $(B_k^{(g)}, A_k^{(g)})_{k=1}^K \leftarrow$ \textsc{SpectralLoRA}$\big(\{(B_k, A_k, r_k)\}_{k=1}^K\big)$\; \tcc*[f]{Step \circled{13}, Alg.~\ref{alg:spectral_lora}}
}
\end{algorithm}

\begin{algorithm}[h!]
\caption{\textsc{SpectralLoRA}: SVD-based Federated LoRA Aggregation (FlexLoRA)}
\label{alg:spectral_lora}
\KwIn{Local LoRA adapters $\{(B_k \in \mathbb{R}^{d \times r_k}, A_k \in \mathbb{R}^{r_k \times d})\}_{k=1}^K$, client weights $\alpha_k = n_k / N$ where $n_k = |\mathcal{D}_k|$, $N = \sum_k n_k$}
\KwOut{Client-specific global adapters $\{(B_k^{(g)}, A_k^{(g)})\}_{k=1}^K$}

\tcc{\textbf{Step 1:} SVD of global weight update}
Compute local updates: $\Delta W_k = B_k A_k \in \mathbb{R}^{d \times d}$ for each client $k$\;

Aggregate: $\Delta W_{\text{agg}} = \sum_{k=1}^{K} \alpha_k \Delta W_k$\;

Perform SVD: $\Delta W_{\text{agg}} = U \Sigma V^T$, where $U \in \mathbb{R}^{d \times d}$, $\Sigma = \text{diag}(\sigma_1, \ldots, \sigma_d)$, $V \in \mathbb{R}^{d \times d}$\;

\tcc{\textbf{Step 2:} Spectral filtering: truncate to client rank}
\For{each client $k = 1, \ldots, K$}{
    $B_k^{(g)} = U_{[:, :r_k]} \, \Sigma_{[:r_k, :r_k]}$\; \tcc*[f]{Top-$r_k$ left singular vectors $\times$ singular values}
    
    $A_k^{(g)} = V^T_{[:r_k, :]}$\; \tcc*[f]{Top-$r_k$ right singular vectors}
}
\Return{$\{(B_k^{(g)}, A_k^{(g)})\}_{k=1}^K$}
\end{algorithm}

\newpage
\section{Proof of Theorem~\ref{th:throughput} (Throughput Scaling)}
\label{app:throughput}

\begin{theorem}{Throughput Scaling}{throughput}
Consider a Split-Federated system with $K$ clients sampled to participate in a
round and a main server, and let the computational cost of the Transformer
blocks be linear with depth. In the server-bound regime $d < d^{\star}$
(Eq.~\eqref{eq:dstar}), as the partition depth ratio $d \to d^{\star}$ the
system throughput $\mathcal{T}$ scales hyperbolically:
\begin{equation}
    \mathcal{T}(d) = \Theta\!\left( \frac{1}{1-d} \right).
\end{equation}
\end{theorem}
\textbf{Example:} Moving the client partition depth ratio from $d=0.5$ to $d=0.9$ yields a theoretical $5\times$ throughput increase, bounded only by the server's sequential processing capacity. (Proof in Appendix~\ref{app:throughput}).

\begin{proof}
Let $W$ denote the total computational work (measured in FLOPs) required for a forward and backward pass on a single data batch for the full $L$-layer model. 
Since a Transformer model is composed of identical blocks (uniform hidden dimension $d_{model}$ and sequence length $S$), the computational cost is linearly proportional to the model depth. For a partition depth ratio $d \in [0, 1)$, we define the work split as:
\begin{align}
    W_{\text{client}} &= d \cdot W \\
    W_{\text{server}} &= (1 - d) \cdot W
\end{align}

The system operates in synchronized rounds. The total time to complete one training round, $T_{\text{round}}$, is determined by the bottleneck between the parallel client execution phase and the sequential server execution phase. Let $\rho_c$ and $\rho_s$ be the computational rates (FLOPs/sec) of a single client device and the central server, respectively.

1. \textbf{Client Phase (Parallel):} All $K$ clients compute their lower-model segments simultaneously.
\begin{equation}
    T_{\text{client}} = \frac{W_{\text{client}}}{\rho_c} = \frac{d \cdot W}{\rho_c}
\end{equation}

2. \textbf{Server Phase (Sequential):} The server must process the activations and gradients for all $K$ clients.
\begin{equation}
    T_{\text{server}} = \frac{K \cdot W_{\text{server}}}{\rho_s} = \frac{K \cdot (1 - d) \cdot W}{\rho_s}
\end{equation}

The system throughput $\mathcal{T}(d)$, defined as the number of client updates processed per second, is given by:
\begin{equation}
    \mathcal{T}(d) = \frac{K}{T_{\text{round}}} = \frac{K}{\max(T_{\text{client}}, T_{\text{server}})}
\end{equation}

% For massive-scale SFF where $K$ is large, the server becomes the bottleneck ($T_{\text{server}} > T_{\text{client}}$). Specifically, this holds when:
% \begin{equation}
%     \frac{K (1-d)}{\rho_s} > \frac{d}{\rho_c} \implies K > \frac{\rho_s}{\rho_c} \frac{d}{1-d}
% \end{equation}
% Under this saturation regime, the throughput is governed by the server's sequential processing time:
% \begin{equation}
%     \mathcal{T}(d) \approx \frac{K}{T_{\text{server}}} = \frac{K}{\frac{K \cdot (1 - d) \cdot W}{\rho_s}} = \frac{\rho_s}{W} \cdot \frac{1}{1 - d}
% \end{equation}

% Since the term $\frac{\rho_s}{W}$ is constant with respect to the partition depth $d$, we conclude that the throughput scales as:
% \begin{equation}
%     \mathcal{T}(d) = \Theta\left( \frac{1}{1 - d} \right)
% \end{equation}
% Thus, as $d \to 1$, the denominator approaches zero, and the throughput diverges hyperbolically.
% \end{proof}

The two phases scale oppositely in $d$: $T_{\text{client}} = dW/\rho_c$ grows
with depth while $T_{\text{server}} = K(1-d)W/\rho_s$ shrinks. The server is
therefore the bottleneck ($T_{\text{server}} > T_{\text{client}}$) precisely when
\begin{equation}
    \frac{K(1-d)}{\rho_s} > \frac{d}{\rho_c}
    \;\;\Longleftrightarrow\;\;
    d < d^{\star} := \frac{K}{K + \rho_s/\rho_c},
    \label{eq:dstar}
\end{equation}
where $d^{\star}$ is the \emph{throughput-optimal cut depth}, the crossover
below which server-side compute dominates the round time. Within this
server-bound regime, throughput is governed by the sequential server phase:
\begin{equation}
    \mathcal{T}(d) \approx \frac{K}{T_{\text{server}}}
    = \frac{K}{\dfrac{K(1-d)W}{\rho_s}}
    = \frac{\rho_s}{W}\cdot\frac{1}{1-d}.
\end{equation}
Since $\rho_s/W$ is constant in $d$, we conclude $\mathcal{T}(d) = \Theta(1/(1-d))$,
diverging hyperbolically as $d \to d^{\star}$. For $d > d^{\star}$ the client
phase dominates and throughput instead \emph{decreases} in $d$; the throughput
incentive for deep cuts is thus confined to the server-bound regime, quantified
next.
\end{proof}

\textbf{Operating regime and the achievable cut range.}
The scaling $\Theta(1/(1-d))$ governs the server-bound regime $d < d^{\star}$,
so whether throughput rises across the \emph{entire} achievable cut range
depends on where $d^{\star}$ sits relative to the deepest valid cut. A valid
split retains at least one server layer, so the maximum cut depth is
$(L{-}1)/L$: $11/12{=}0.91$, $15/16{=}0.94$, and $35/36{=}0.97$ for GPT-2
Small, Medium, and Large. Table~\ref{tab:dstar} reports
$d^{\star} = K/(K+\rho_s/\rho_c)$ across participation levels $K$ and compute
ratios $\rho_s/\rho_c$. When $d^{\star}$ meets or exceeds $(L{-}1)/L$, the
server is the bottleneck at every cut and throughput increases monotonically
with depth, as plotted in Figure~\ref{fig:depth_dilemma}; for smaller
$d^{\star}$, throughput peaks at $d^{\star}$ and the deep-cut incentive instead
rests on privacy (below).

\textbf{The dilemma holds across the operating range.}
For the Depth-Performance Dilemma to exist, $d^{\star}$ must land within the
attention-collapse regime, which begins at $50$--$60\%$ depth for all models
(Figure~\ref{fig:rank_collapse}). This holds across realistic settings: at
$\rho_s/\rho_c{=}10$, $K{\geq}10$ already gives $d^{\star}{\geq}0.5$ (our
$100$-client GLUE setup samples $K{=}10$/round); at $\rho_s/\rho_c{=}2$, even
$K{=}3$ gives $d^{\star}{=}0.60$; and at $\rho_s/\rho_c{=}1$ (server and clients
equally capable), our E2E/GSM8K setup ($K{=}3$, $d^{\star}{=}0.75$) and GLUE
setup ($K{=}10$, $d^{\star}{=}0.91$) both fall inside the collapse regime.
Crucially, the dilemma does not rest on throughput alone: even where throughput
favors shallow cuts (e.g., $d^{\star}{=}0.23$ at $K{=}3$, $\rho_s/\rho_c{=}10$)
or the server is no longer the bottleneck ($d > d^{\star}$), privacy leakage
still decays consistently with depth (Figure~\ref{fig:ep_vs_depth}), independent of
$K$, $\rho$, and execution model, so a privacy-driven deployment cuts deep at
any $K$.
\begin{table}[t!]
\centering
\small
\caption{\textbf{Throughput-optimal cut depth $d^{\star}=K/(K+\rho_s/\rho_c)$}
as a function of the participating clients per round $K$ and the
server-to-client compute ratio $\rho_s/\rho_c$ ($\rho_s/\rho_c{=}1$: equally
capable server and clients). The server stays the bottleneck for all $d<d^{\star}$;
the Depth-Performance Dilemma holds whenever $d^{\star}$ reaches the
attention-collapse onset ($\approx 50$--$60\%$ depth).}
\label{tab:dstar}
\setlength{\tabcolsep}{10pt}
\begin{tabular}{cccc}
\toprule
$K$ & $d^{\star}\,(\rho_s/\rho_c{=}10)$ & $d^{\star}\,(\rho_s/\rho_c{=}2)$ & $d^{\star}\,(\rho_s/\rho_c{=}1)$ \\
\midrule
3   & 0.23 & 0.60 & 0.75 \\
10  & 0.50 & 0.83 & 0.91 \\
30  & 0.75 & 0.94 & 0.97 \\
100 & 0.91 & 0.98 & 0.99 \\
\bottomrule
\end{tabular}
\end{table}

\newpage
\section{Aggregation Noise Bounds}
\label{app:noise-bound}
\begin{assumption}{Heterogeneous Complexity}{hetero-assume}
Assume,  disjoint clients $k \in \{1 \dots K\}$ possess optimal updates $\Delta W_k^*$ associated with distinct local tasks. We define the Global Noise $\mathcal{E}$ as the deviation of the aggregated update $\Delta W_{agg}$ from the ideal multi-task manifold.
\end{assumption}

%We present the \textbf{Aggregation Noise Bounds}, proving that all existing methods introduce systemic error under non-IID conditions (refer Fed Server in Figure~\ref{fig:SFF}).
\begin{theorem}{Aggregation Noise}{agg_noise}
%\label{thm:noise_trilemma}
Let $\Delta W_k = B_k A_k$ be the update from client $k$. The aggregation error $\mathcal{E}$ for each paradigm is lower-bounded by specific noise artifacts:

\begin{enumerate}[leftmargin=*]
    \item \textbf{\textsc{Average} ($\mathcal{E}_{avg}$): The Variance Noise.}\\
    Averaging forces a common rank $r_{max}$. Due to the nonlinearity of matrix multiplication ($\mathbb{E}[B]\mathbb{E}[A] \neq \mathbb{E}[BA]$) and the incoherence of client updates, the aggregate spectrum collapses:
    \begin{equation}
        \mathcal{E}_{avg} \approx \underbrace{\text{Cov}(B, A)}_{\text{Algebraic Error}} + \underbrace{\frac{1}{\sqrt{K}}\|\Sigma_{noise}\|_F}_{\text{Spectrum Flattening}}
    \end{equation}
    The resulting update resembles high-rank ``white noise'', diluting specific task signals.

    % \item \textbf{Stacking / FLoRA ($\mathcal{E}_{stack}$): The Interference Noise.}
    % Summing updates $\Delta W_{stack} = \sum_{k} B_k A_k$ introduces \textbf{Destructive Interference}. For a specific input $x$ belonging to task $k$, the updates from other clients $j \neq k$ act as adversarial perturbations:
    % \begin{equation}
    %     \mathcal{E}_{stack} \propto \left\| \sum_{j \neq k} x B_j A_j \right\|_F \quad (\text{Cross-Task Interference})
    % \end{equation}
    % In deep split networks, where $K$ is large, this ``crosstalk'' variance ($\sigma^2_{interference}$) overwhelms the local signal magnitude, preventing stable convergence.

    \item \textbf{\textsc{Stack} ($\mathcal{E}_{stack}$): The Interference Noise.}
    Aggregating via summation ($\Delta W_{stack} = \sum B_k A_k$) creates a ``Cross-Task Interference'' problem. For an input $x$ belonging to Client $k$, the adapters of all other clients $j \neq k$ act as adversarial perturbations:
    % \begin{equation}
    %     y = \hspace{-1.5em}\underbrace{x W_{base}}_{\text{Pre-trained Knowledge}}\hspace{-1.4em}+\underbrace{x B_k A_k}_{\text{Signal}} + \hspace{-1.6em}\underbrace{\sum_{j \neq k} x B_j A_j}_{\text{Cross-Talk (interference)}}
    % \end{equation}
    \begin{equation}
        \mathcal{E}_{stack} = \| \sum_{j \neq k} x B_j A_j \|_F \propto \sqrt{K-1} \cdot \sigma_{cross} %\quad (\text{Cross-Task Interference})
    \end{equation}
    In deep split networks ($K \gg 1$), this crosstalk variance overwhelms the local signal magnitude.
    %The noise term $\mathcal{E}_{stack} = \| \sum_{j \neq k} x B_j A_j \|$ scales with $\sqrt{K}$ (assuming uncorrelated tasks), overwhelming the signal as the number of clients increases.

    \item \textbf{\textsc{Freeze} ($\mathcal{E}_{freeze}$): The Bias Noise.}\\
    % Fixing a shared projection $A_{fix}$ forces all clients to optimize within a restricted subspace. The error is the unlearnable residual energy:
    % \begin{equation}
    %     \mathcal{E}_{freeze} \geq \hspace{-0.1em} \sum_{k=1}^K \| \Delta W_k^* - \text{proj}_{A_{fix}}\hspace{-0.1em}(\Delta W_k^*) \|_F^2
    % \end{equation}
    % This constitutes an irreducible \textbf{Approximation Bias} for any client whose optimal rank subspace is orthogonal to $A_{fix}$.
    Fixing a shared projection $A_{fixed}$ constraints the global update $\Delta W_{freeze} = \bar{B} A_{fixed}$ to a fixed subspace $\mathcal{S}_A$. The error is the unlearnable residual energy of the ideal global update:
    \begin{equation}
        \mathcal{E}_{freeze} = \| \Delta W^*_{ideal} - \text{proj}_{\mathcal{S}_A}(\Delta W^*_{ideal}) \|_F
    \end{equation}
    This constitutes an irreducible \textbf{Approximation Bias} for any task features lying in the orthogonal complement $\mathcal{S}_A^\perp$.

\end{enumerate}
\end{theorem}

\subsection{Proof of Theorem~\ref{th:agg_noise} (The Aggregation Noise)}
\label{app:proof_trilemma}

In this section, we rigorously derive the error bounds for the three dominant LoRA aggregation paradigms: \textsc{Average} (FedIT~\cite{bai2024federated}), \textsc{Stack} (FLoRA~\cite{flora}), and \textsc{Freeze} (FFA-LoRA~\cite{sun2024improving}).

% \subsection{Setup and Definitions}
% Let there be $K$ clients. Each client $k$ computes an optimal local update $\Delta W_k = B_k A_k$, where $B_k \in \mathbb{R}^{d \times r}$ and $A_k \in \mathbb{R}^{r \times d}$.
% We assume the global goal is to learn a representation that minimizes the loss across all client distributions. 

% Let $x$ be an input token belonging to the distribution of Client $k$ (denoted $x^{(k)}$). The \textbf{Ideal Output} for this token is:
% \begin{equation}
%     y_{ideal} = x^{(k)} (W_{base} + B_k A_k)
% \end{equation}
% The \textbf{Aggregation Error} $\mathcal{E}$ is defined as the expected deviation of the aggregated model's output from this ideal output:
% \begin{equation}
%     \mathcal{E} = \mathbb{E}_{k} \left[ \| y_{agg} - y_{ideal} \|_F \right]
% \end{equation}

% ---

\subsubsection{Setup and Definitions}
Let there be $K$ clients. Each client $k$ possesses an optimal update $\Delta W_k^*$ required to solve their local task.
The \textbf{Ideal Global Update} is the average of these optimal directions (assuming a uniform objective):
\begin{equation}
    \Delta W^*_{ideal} = \frac{1}{K} \sum_{k=1}^K \Delta W_k^*
\end{equation}
The \textbf{Aggregation Error} $\mathcal{E}$ is the Frobenius norm of the difference between the actual aggregated update and this ideal update.

\subsubsection{Part 1: \textsc{Average} (FedIT~\cite{bai2024federated}) (The Variance Noise)}
\textbf{Mechanism:} The server computes $\bar{B} = \frac{1}{K}\sum B_k$ and $\bar{A} = \frac{1}{K}\sum A_k$. The aggregated update is $\Delta W_{avg} = \bar{B}\bar{A}$.

\textbf{Proof:}
Expanding the product of sums:
\begin{align}
    \Delta W_{avg} &= \left( \frac{1}{K} \sum_{i} B_i \right) \left( \frac{1}{K} \sum_{j} A_j \right) \\
    &= \frac{1}{K^2} \sum_{i} B_i A_i + \frac{1}{K^2} \sum_{i \neq j} B_i A_j \\
    &= \frac{1}{K} \Delta W^*_{ideal} + \underbrace{\frac{1}{K^2} \sum_{i \neq j} B_i A_j}_{\text{Algebraic Noise}}
\end{align}
where the first equality uses $\frac{1}{K^2}\sum_i B_i A_i = \frac{1}{K}\cdot\frac{1}{K}\sum_i \Delta W_i^* = \frac{1}{K}\Delta W^*_{ideal}$.
The signal term therefore scales as $O(1/K)$ relative to each client's update magnitude.
%The error is dominated by the cross-terms. 
% Under the assumption that clients are non-IID (uncorrelated subspaces), the cross-terms $B_i A_j$ are incoherent. The singular value spectrum of this sum tends toward a uniform distribution ("white noise"), diluting the task-specific spectral signatures.
% Under the assumption that client adapters are statistically independent and isotropic (randomly oriented in high dimensions due to non-IID tasks), the cross-terms $B_i A_j$ behave like random noise. By the law of large numbers and random matrix theory, the singular value spectrum of such a sum of random matrices tends toward a Marchenko-Pastur distribution (or ``white noise'' spectrum), where:
% \begin{equation}
%     \sigma_{noise} \propto \frac{1}{\sqrt{K}}
% \end{equation}
% Thus, averaging effectively ``flattens'' the sharp, task-specific features of $B_k A_k$ into a high-rank, low-magnitude noise floor.

\textbf{Spectral Analysis (Justification of ``Flattening''):}
We analyze the spectral norm of the algebraic noise term $N = \frac{1}{K^2} \sum_{i \neq j} B_i A_j$.
Assuming the cross-terms $B_i A_j$ are independent isotropic variables with variance $\sigma^2$, the sum consists of $K(K-1)$ terms. By the universality of random matrices (Marchenko-Pastur law), the singular values of $N$ converge to a bulk distribution bounded by:
\begin{equation}
    \| N \|_2 \approx \frac{1}{K^2} \cdot \sqrt{K(K-1)} \cdot \sigma \approx \frac{\sigma}{\sqrt{K}}
\end{equation}
Crucially, this noise energy is not concentrated in a specific direction (like the signal) but is spread across all dimensions $d$.
This transforms the spectral profile of the weight update:
\begin{itemize}
    \item \textbf{Before Averaging:} $\Delta W_k$ is low-rank (spiky spectrum).
    \item \textbf{After Averaging:} $\Delta W_{avg}$ is high-rank (flat spectrum).
\end{itemize}
We term this phenomenon \textbf{Spectrum Flattening}, as the noise floor rises to $\propto 1/\sqrt{K}$, obscuring the singular values of the true signal. Thus, averaging effectively ``flattens'' the sharp, task-specific features of $B_k A_k$ into a high-rank, low-magnitude noise floor.

\subsubsection{Part 2: \textsc{Stack} (FLoRA~\cite{flora}) (The Interference Noise)}
\textbf{Mechanism:} The server maintains the sum of all adapters. For an input $x$, the forward pass is $x W_{base} + \sum_{j=1}^K x B_j A_j$.

% \textbf{Proof:}
% Let the input $x^{(k)}$ belong to Client $k$. The ideal update is $x^{(k)} B_k A_k$. The actual output is:
% \begin{equation}
%     y_{stack} = x^{(k)} W_{base} + x^{(k)} B_k A_k + \sum_{j \neq k} x^{(k)} B_j A_j
% \end{equation}
% The error term is the \textbf{Interference Sum}:
% \begin{equation}
%     \mathcal{E}_{stack} = \left\| \sum_{j \neq k} x^{(k)} B_j A_j \right\|_F
% \end{equation}
% We assume that the adapters of other clients ($j \neq k$) are uncorrelated with the current task $k$. Let $z_j = x^{(k)} B_j A_j$ be the interference vector from client $j$. We assume $\mathbb{E}[z_j] = 0$ (zero-mean noise relative to task $k$) and variance $\text{Var}(z_j) = \sigma^2$.
% The squared norm of the sum of uncorrelated vectors is the sum of their squared norms:
% \begin{equation}
%     \mathbb{E} \left[ \left\| \sum_{j \neq k} z_j \right\|^2 \right] = \sum_{j \neq k} \mathbb{E} [\| z_j \|^2] \approx (K-1) \sigma^2
% \end{equation}
% Thus, the Root Mean Square (RMS) error scales as:
% \begin{equation}
%     \mathcal{E}_{stack} \propto \sqrt{K} \cdot \sigma
% \end{equation}
% As the number of clients $K$ increases, the magnitude of the interference (``Cross-Talk'') grows, eventually overwhelming the signal from the single correct adapter ($B_k A_k$).

% ---
\textbf{Mechanism:} The server sums the updates $\Delta W_{stack} = \sum_{k} B_k A_k$.
During inference, for an input $x$ belonging to Client $k$, the model computes:
\begin{equation}
    y = x W_{base} + \frac{1}{K}\sum_{j=1}^K x B_j A_j
\end{equation}

\textbf{Proof:}
The ideal activation for input $x^{(k)}$ is $x^{(k)} W_{base} + x^{(k)} B_k A_k$.
The error is the \textbf{Interference Sum} from all other clients:
\begin{equation}
    \mathcal{E}_{stack} = \left\| \frac{1}{K}\sum_{j \neq k} x^{(k)} B_j A_j \right\|_F
\end{equation}
Modeling the interference terms $z_j = x^{(k)} B_j A_j$ as independent random vectors with variance $\sigma^2$ (relative to the subspace of task $k$), the total interference variance is:
\begin{equation}
    \text{Var}(\mathcal{E}_{stack}) \approx (1-1/K) \sigma^2
\end{equation}
Thus, the Root Mean Square (RMS) error scales as:
\begin{equation}
    \mathcal{E}_{stack} \propto \frac{\sigma}{\sqrt{K}}
\end{equation}

Thus, the noise magnitude scales with $(\sqrt{K})^{-1}$.  Signal-to-Noise Ratio (SNR) vanishes as $K \to \infty$  because the signal shrinks faster than the noise.

%As the number of clients $K$ increases, the magnitude of the interference (``Cross-Talk'') grows, eventually overwhelming the signal from the single correct adapter.
%In deep split settings with hundreds of clients, this "Cross-Talk" overwhelms the signal of the single correct adapter.

\subsubsection{Part 3: \textsc{Freeze} (FFA-LoRA~\cite{sun2024improving}) (The Bias Noise)}
% \textbf{Mechanism:} A shared matrix $A_{fix}$ is frozen at initialization. Clients only update $B_k$. The update is constrained to $\Delta W_k = B_k A_{fix}$.

% \textbf{Proof:}
% This is a standard linear subspace projection problem. The optimization problem for client $k$ becomes:
% \begin{equation}
%     \min_{B_k} \| \Delta W_k^* - B_k A_{fix} \|_F^2
% \end{equation}
% Let $\mathcal{S}_A$ be the row space spanned by $A_{fix}$. The optimal solution for $B_k$ yields the orthogonal projection of the true update $\Delta W_k^*$ onto $\mathcal{S}_A$:
% \begin{equation}
%     \text{proj}_{\mathcal{S}_A}(\Delta W_k^*) = B_k^{opt} A_{fixed}
% \end{equation}
% By the Projection Theorem (Pythagorean Theorem for matrices), the squared error is the energy of the update that lies in the orthogonal complement $\mathcal{S}_A^\perp$:
% \begin{equation}
%     \mathcal{E}_{freeze}^2 = \| \Delta W_k^* - \text{proj}_{\mathcal{S}_A}(\Delta W_k^*) \|_F^2 = \| \mathcal{P}_{\mathcal{S}_A^\perp} (\Delta W_k^*) \|_F^2
% \end{equation}
% If the optimal updates $\Delta W_k^*$ for different tasks lie in disjoint subspaces (which is true for heterogeneous tasks), then for any fixed $A$ of rank $r \ll d$, there exists a significant portion of energy in $\mathcal{S}_A^\perp$. This error is irreducible regardless of training time.

% \hfill \qedsymbol

\textbf{Mechanism:} A projection matrix $A_{fixed}$ is frozen. Clients optimize $B_k$. The global model is aggregated as:
\begin{equation}
    \Delta W_{freeze} = \left( \frac{1}{K} \sum_{k=1}^K B_k \right) A_{fixed} = \bar{B} A_{fixed}
\end{equation}

\textbf{Proof:}
Let $\mathcal{S}_A$ be the row space spanned by the fixed matrix $A_{fixed}$ (rank $r$). By definition, the global update $\Delta W_{freeze}$ lies entirely within $\mathcal{S}_A$.
However, the \textbf{Ideal Global Update} $\Delta W^*_{ideal} = \frac{1}{K} \sum \Delta W_k^*$ is a summation of $K$ diverse task updates. If tasks are heterogeneous, their principal components span a union of subspaces with effective rank $R_{eff} \gg r$.
The error is the projection loss of the ideal update onto the restricted subspace $\mathcal{S}_A$:
\begin{equation}
    \mathcal{E}_{freeze} = \| \Delta W^*_{ideal} - \text{proj}_{\mathcal{S}_A}(\Delta W^*_{ideal}) \|_F
\end{equation}
Substituting the sum:
\begin{equation}
    \mathcal{E}_{freeze} = \left\| \frac{1}{K} \sum_{k=1}^K \underbrace{\left( \Delta W_k^* - B_k^{opt} A_{fixed} \right)}_{\text{Local Residuals } r_k} \right\|_F
\end{equation}
Since the local residuals $r_k$ (the parts of the tasks orthogonal to $A_{fixed}$) are unlikely to cancel out perfectly in a non-IID setting, this constitutes an irreducible \textbf{Approximation Bias}. The global model effectively "blind" to any task features that lie in the orthogonal complement $\mathcal{S}_A^\perp$.

\hfill \qedsymbol

\newpage
\section{ Analysis of Spectral Denoising}

\begin{theorem}{Spectral Denoising}{spectral_denoising}
Spectral Truncation improves the Signal-to-Noise Ratio (SNR) by projecting orthogonal noise out of the signal subspace.
\end{theorem}

\begin{proof}
Let the true signal be $S$ (rank $r$) and the noise be $N$ (rank $d$, resulting from incoherent aggregation).
We assume a standard signal processing model where noise is isotropic (spherical): $N = \epsilon G$, where $G$ has i.i.d. Gaussian entries and $\epsilon$ is the noise magnitude.

\textbf{1. Energy Distribution:}
The total energy of the noise is $\mathbb{E}[\|N\|_F^2] = d^2 \epsilon^2$.
However, because the noise is high-rank (Theorem \ref{th:spectral_denoising}), this energy is spread roughly equally across all $d$ dimensions of the space.
The energy contained within any $r$-dimensional subspace is proportional to the fraction of dimensions:
$$ E_{projected} \approx \frac{r}{d} E_{total} = \frac{r}{d} (d^2 \epsilon^2) = r d \epsilon^2 $$

\textbf{2. The Filtering Operation:}
Spectral LoRA projects the data onto the top-$r$ singular vectors. Since $S$ is rank-$r$, it lies almost entirely within this subspace (assuming signal $\gg$ noise).
The retained noise is only the component of $N$ that aligns with $S$.
The retained noise energy is reduced by a factor of $\frac{r}{d}$.

\textbf{3. SNR Improvement:}
Let $\text{SNR}_{in} = \frac{\|S\|_F^2}{\|N\|_F^2}$.
The output SNR after SVD is:
$$ \text{SNR}_{out} = \frac{\|S\|_F^2}{\|N_{projected}\|_F^2} \approx \frac{\|S\|_F^2}{\frac{r}{d} \|N\|_F^2} = \frac{d}{r} \text{SNR}_{in} $$
For a typical LLaMA model ($d=4096$) and LoRA rank ($r=16$), the SNR improvement is $\frac{4096}{16} = 256\times$.
This theoretical gain explains why Spectral LoRA prevents the degradation of performance at deep cuts, where naive averaging fails.
\end{proof}

\newpage
\section{Complexity Analysis}
\label{app:complexity_analysis}
\begin{table}[h]
\caption{\textbf{Computational complexity in Split-Federated Fine-Tuning (SFF).}
Per-epoch local training cost is denoted by $\mathcal{T}(\cdot)$ and scales linearly with the number of trainable Transformer blocks.
$\ell_c$ is the client-side cut depth (number of blocks executed on the client), and $L-\ell_c$ blocks are executed on the main server.
$m$ is the embedding size, $n$ is the context length,  $|D_k|$ is the number of local training samples, $K$ is the number of clients, $r_k$ is the LoRA rank on client $k$, $r=\sum_{k=1}^K r_k$, and $r_s$ denotes the LoRA rank used for server-side adapters.
Fed server costs correspond to aggregating \emph{client-side} LoRA (replace $\ell_c$ by $L$ if LoRA is applied to all layers).}
\label{tab:comp-cost-sff}
\begin{center}
\begin{small}
\begin{sc}
\resizebox{\textwidth}{!}{
\begin{tabular}{lccc}
\toprule
Method & Client & Main Server (split tail) & Fed Server (aggregation) \\
\midrule
Full FT &
$\mathcal{O}\!\left(E \cdot \mathcal{T}(m,n,|D_k|;\ell_c)\right)$ &
$\mathcal{O}\!\left(E \cdot \mathcal{T}(m,n,r_s,|D_k|;L-\ell_c)\right)$ &
$\mathcal{O}(l_cK m n)$ \\

Average &
$\mathcal{O}\!\left(E \cdot \mathcal{T}(m,n,r_k,|D_k|;\ell_c)\right)$ &
$\mathcal{O}\!\left(E \cdot \mathcal{T}(m,n,r_s,|D_k|;L-\ell_c)\right)$ &
$\mathcal{O}\!\left(\ell_c (m+n)\, r\right)$ \\

Freeze &
$\mathcal{O}\!\left(E \cdot \mathcal{T}(m,n,r_k,|D_k|;\ell_c)\right)$ &
$\mathcal{O}\!\left(E \cdot \mathcal{T}(m,n,r_s,|D_k|;L-\ell_c)\right)$ &
$\mathcal{O}\!\left(\ell_c\, n\, r\right)$ \\

Stack &
$\mathcal{O}\!\left(E \cdot \mathcal{T}(m,n,r_k,|D_k|;\ell_c)\right)
+ \mathcal{O}\!\left(\ell_c\, m\, n \sum_{k=1}^K r_k\right)$ &
$\mathcal{O}\!\left(E \cdot \mathcal{T}(m,n,r_s,|D_k|;L-\ell_c)\right)$ &
None \\

SVD &
$\mathcal{O}\!\left(E \cdot \mathcal{T}(m,n,r_k,|D_k|;\ell_c)\right)$ &
$\mathcal{O}\!\left(E \cdot \mathcal{T}(m,n,r_s,|D_k|;L-\ell_c)\right)$ &
$\mathcal{O}\!\left(\ell_c\, m n\, r\right)
+ \mathcal{O}\!\left(\ell_c \min(m,n)\, m n\right)
+ \mathcal{O}\!\left(\ell_c\, m\, r^2\right)$ \\
\bottomrule
\end{tabular}
}
\end{sc}
\end{small}
\end{center}
\end{table}

\begin{table}[h]
\caption{\textbf{Communication overhead per round in Split-Federated Fine-Tuning (SFF).}
We decompose communication into (i) federated LoRA parameter exchange between clients
and the Fed server, and (ii) bidirectional split-learning traffic between clients
and the main server.
$I$ denotes the number of local client optimization steps per round, B is the batch size; each step
incurs transmission of cut activations and gradients of size $\Theta(nm)$. \textbf{Observation:} Split-learning traffic scales with the full activation dimension
$\Theta(Bnm)$ and typically dominates communication cost, exceeding LoRA parameter exchange with fed server by orders of magnitude for practical ranks.}

\label{tab:comm-cost-sff}
\begin{center}
\begin{small}
\begin{sc}
\resizebox{\textwidth}{!}{
\begin{tabular}{lccc}
\toprule
Method &
Client$\rightarrow$Fed (Upload LoRA) &
Fed$\rightarrow$Client (Download LoRA) &
Client$\leftrightarrow$Main (Split traffic) \\
\midrule
Full FT &
$\mathcal{O}\!\left(\ell_c K \cdot m n\right)$ &
$\mathcal{O}\!\left(\ell_c K \cdot m n\right)$ &
$\mathcal{O}\!\left(K I B \cdot n m\right)$ \\

Average &
$\mathcal{O}\!\left(\ell_c (m+n)\, r\right)$ &
$\mathcal{O}\!\left(\ell_c K (m+n)\, \max(r_k\right))$ &
$\mathcal{O}\!\left(K I B \cdot n m\right)$ \\

Freeze &
$\mathcal{O}\!\left(\ell_c\, n\, r\right)$ &
$\mathcal{O}\!\left(\ell_c K\, n\, \max(r_k\right))$ &
$\mathcal{O}\!\left(K I B \cdot n m\right)$ \\

Stack &
$\mathcal{O}\!\left(\ell_c (m+n)\, r\right)$ &
$\mathcal{O}\!\left(\ell_c K (m+n)\, r\right)$ &
$\mathcal{O}\!\left(K I B \cdot n m\right)$ \\

SVD &
$\mathcal{O}\!\left(\ell_c (m+n)\, r\right)$ &
$\mathcal{O}\!\left(\ell_c (m+n)\, r\right)$ &
$\mathcal{O}\!\left(K I B \cdot n m\right)$ \\
\bottomrule
\end{tabular}
}
\end{sc}
\end{small}
\end{center}
\end{table}

\newpage
\section{Empirical Privacy Evaluation via Inversion Attacks}
\label{app:eia}
 
We quantify information leakage at each cut layer by training adversarial
inversion models that attempt to reconstruct input tokens from the
cut-layer hidden states exposed to the server.
 
\paragraph{Setup.}
We evaluate GPT-2 Small (12 layers, hidden size 768), Medium (24 layers,
1024), and Large (36 layers, 1280) under a split inference setting. Given
an input sequence $x = (x_1, \dots, x_T)$, the client computes the
intermediate representation at cut layer $\ell_c$:
\begin{equation}
    \mathbf{z}^{(\ell_c)} = f_{\leq \ell_c}(x) \in \mathbb{R}^{T \times D},
\end{equation}
where $f_{\leq \ell_c}$ denotes the Transformer truncated at layer $\ell_c$
and $D$ is the hidden dimension. These representations are treated as the
observable interface accessible to an adversary, the sole information
the server receives about the client's private input.
 
\paragraph{Inversion model.}
For each cut layer, we train a separate adversarial inversion model
$g_\theta : \mathbb{R}^D \to \{1, \dots, V\}$, implemented as a token-wise
MLP with an input projection ($D \to H$), a ReLU non-linearity, and an
output layer ($H \to V$), where $V$ is the vocabulary size. Training data
consists of aligned representation--token pairs:
$\mathcal{D}_{\ell_c} = \{(\mathbf{z}_t^{(\ell_c)}, x_t)\}_{t=1}^N$,
constructed by passing sequences through the frozen truncated model and
removing padding tokens. The model is trained with cross-entropy loss:
\begin{equation}
    \mathcal{L} = -\mathbb{E}_{(\mathbf{z}, x) \sim \mathcal{D}_{\ell_c}}
    \log P_\theta(x \mid \mathbf{z}),
\end{equation}
using AdamW for a fixed number of epochs. This setup assumes a strong
adversary with access to representation--token pairs and knowledge of the
data distribution, a worst-case leakage scenario.
 
\paragraph{Metrics.}
After training, the inversion model is evaluated on the same representation
distribution. We measure the \emph{Recovery Rate} (RR), defined as token
reconstruction accuracy:
\begin{equation}
    \mathrm{RR} = \frac{1}{N} \sum_{i=1}^{N}
    \mathbf{1}[\hat{x}_i = x_i], \quad
    \hat{x}_i = \arg\max_{v} P_\theta(v \mid \mathbf{z}_i).
\end{equation}
We report \emph{Empirical Privacy} EP $= (1 - \mathrm{RR}) \times 100$,
normalized to $[0, 100]$. A score of 0 indicates perfect token
reconstruction by the attacker (no privacy); a score of 100 indicates
complete failure of reconstruction (full privacy). This metric is computed
independently for each cut layer, yielding the depth-vs-privacy curves
reported in Figure~\ref{fig:ep_vs_depth}.

 \begin{figure}[b]
    \centering
\includegraphics[width=0.50\linewidth]{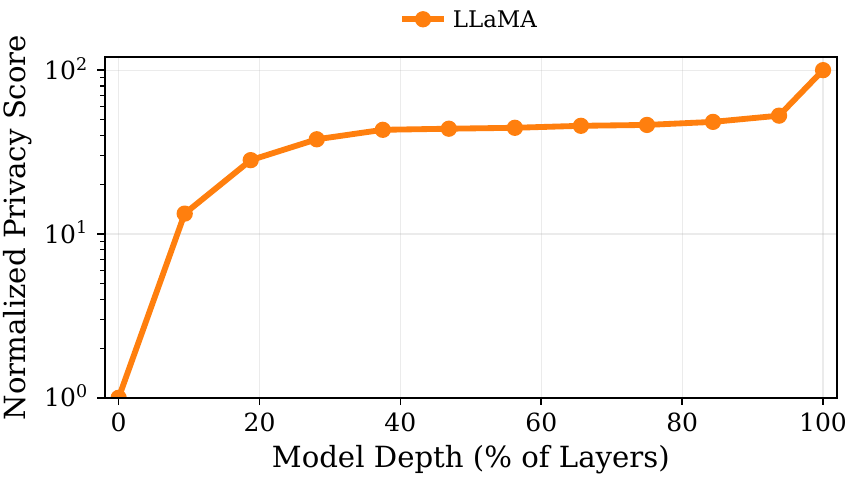}
    \caption{\textbf{Normalized Privacy Score vs.\ partition depth for
    Llama-3-8B-Instruct.} Score${=}(1{-}\mathrm{RR}){\times}100$ (log-scale
    $y$-axis), where RR is the token recovery rate of the adversarial MLP
    inversion model. As with the GPT-2 family (Figure~\ref{fig:ep_vs_depth}),
    leakage decays with depth. %: the score rises steeply through the early layers and saturates at 100 at the deepest cut, confirming the depth--privacy relationship holds at 8B scale.
    }
    \label{fig:ep_vs_depth_llama}
\end{figure}

\paragraph{Extension to Llama-3-8B.}
We repeat the inversion protocol on Llama-3-8B-Instruct (32 layers, hidden
size 4096) using the same MLP attacker and worst-case white-box assumptions.
Figure~\ref{fig:ep_vs_depth_llama} shows the same qualitative trend observed
for the GPT-2 family: the normalized privacy score increases with depth,
rising steeply through the early layers, plateauing across the middle depths,
and saturating at 100 at the deepest cut. This confirms that the
depth--privacy relationship is not an artifact of the GPT-2 architecture but
holds at 8B scale, consistent with the data-processing view that successive
Transformer blocks can only reduce the mutual information $I(x;\mathbf{h}_{\ell_c})$
between the cut-layer representation and the private input.

\newpage
\section{Additional Experiments}
\label{app:add_exp}

\begin{figure*}[h!]
    \centering
    
    \begin{subfigure}[t]{0.48\linewidth}
        \centering
        \includegraphics[width=\linewidth]{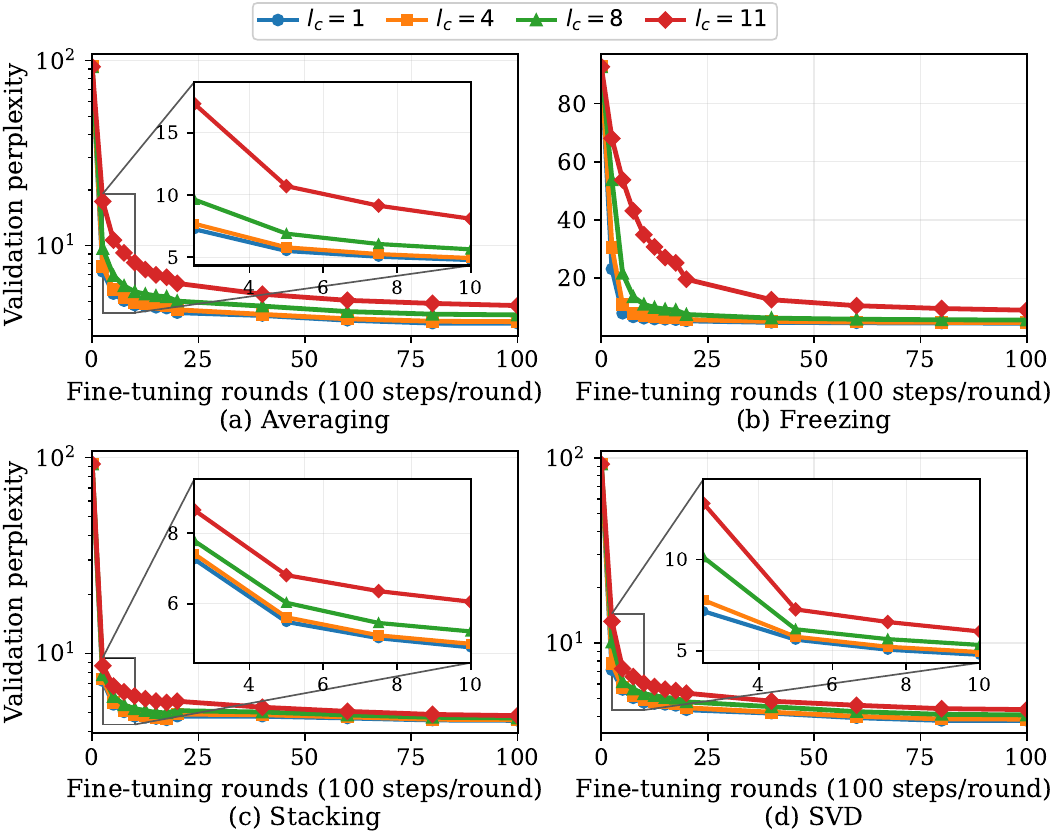}
        \caption{GPT-2 Small}
        \label{fig:ppl_convergence_sm}
    \end{subfigure}
\hfill
    \begin{subfigure}[t]{0.48\linewidth}
        \centering
        \includegraphics[width=\linewidth]{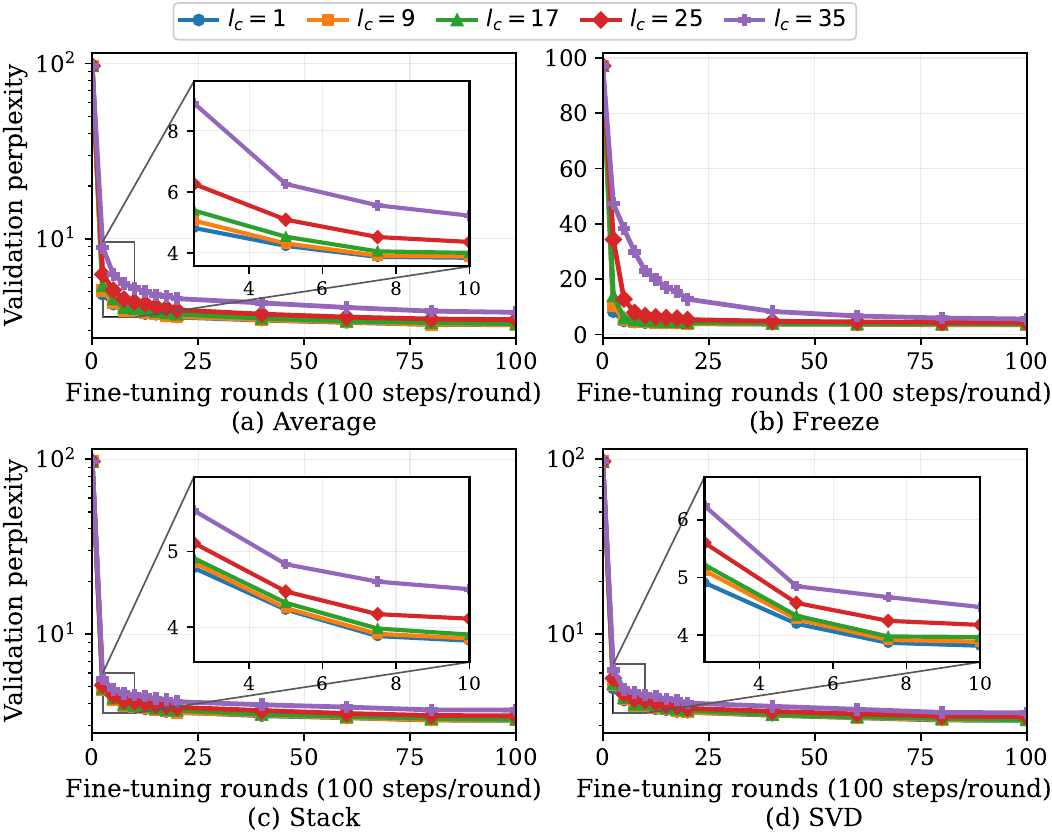}
        \caption{GPT-2 Large}
        \label{fig:ppl_convergence_s}
    \end{subfigure}
    \caption{\textbf{Convergence of GPT-2 Small and Large in heterogeneous setting}
Validation perplexity (PPL; lower is better) vs.\ rounds (100 steps/round) for cut depth $l_c$.
Main panels use log-$y$ (except Freezing in (b), linear-$y$); insets (when shown) zoom early training with linear-$y$.
\textbf{Observation:} Deeper cuts slow convergence and worsen final PPL; Stacking and SVD are consistently more stable than Averaging and Freezing.}

    \label{fig:convergence-GPT2S-L}
\end{figure*}

\begin{table*}[h]
\caption{Comparison of average E2E scores of GPT2-S, GPT2-M and GPT2-L across cutlayers for ranks 8 and 16 under homogenous settings. We report a unified E2E score computed as the arithmetic mean of per-metric min–max normalized scores across ROUGE-L, METEOR, BLEU, CIDEr, and NIST.}
\label{tab:depth_audit}
\centering
\small
\begin{sc}
\setlength{\tabcolsep}{3.5pt}     % Tightens internal columns slightly to fit page

\begin{tabular}{l cccc @{\hskip 0.2in} cccc @{\hskip 0.2in} cccc}
\toprule
% Top Header
& \multicolumn{4}{c}{\textbf{Shallow Cut} ($L/4$)} 
& \multicolumn{4}{c}{\textbf{Middle Cut} ($L/2$)} 
& \multicolumn{4}{c}{\textbf{Deep Cut} ($3L/4$)} \\

% Spanning Lines (The 'lr' trims them so they don't touch)

% Sub-headers
\textbf{Model} 
& \textbf{Avg} & \textbf{Frz} & \textbf{Stk} & \textbf{SVD} 
& \textbf{Avg} & \textbf{Frz} & \textbf{Stk} & \textbf{SVD} 
& \textbf{Avg} & \textbf{Frz} & \textbf{Stk} & \textbf{SVD} \\
\midrule

% Data Rows
\textbf{GPT2-S} ($r\!=\!8$)  
& 57.4 & 53.3 & 57.2 & \textbf{57.9} 
& 57.9 & \textbf{58.6} & 53.1 & 58.0 
& \textbf{57.9} & 56.7 & 53.5 & 56.1 \\

\textbf{GPT2-S} ($r\!=\!16$) 
& \textbf{57.1} & 53.9 & 56.5 & 56.0 
& 56.0 & 56.4 & 53.5 & \textbf{58.0} 
& 56.6 & \textbf{56.8} & 53.0 & 52.8 \\

\addlinespace[0.5em] % Adds a subtle gap between Model Sizes

\textbf{GPT2-M} ($r\!=\!8$)  
& \textbf{60.3} & 56.2 & 58.5 & 59.8 
& 59.8 & \textbf{59.9} & 56.3 & 57.7 
& \textbf{60.4} & 59.7 & 56.6 & 59.2 \\

\textbf{GPT2-M} ($r\!=\!16$) 
& 58.4 & 55.9 & 58.5 & \textbf{58.7} 
& \textbf{58.7} & 57.6 & 56.0 & 57.4 
& \textbf{57.7} & 57.5 & 56.5 & 57.6 \\

\bottomrule
\end{tabular}
\end{sc}
\end{table*}

\newpage
% The NeurIPS paper checklist is optional in a preprint. Keep it, or comment
% out the line below if you prefer a checklist-free arXiv version.
% \input{checklist.tex}
\end{document}